\documentclass{article}
\pdfoutput=1
\usepackage{times}
\usepackage{graphicx} 
\usepackage{subfigure} 

\usepackage{natbib}

\usepackage{algorithm}
\usepackage{algorithmic}

\usepackage{hyperref}

\usepackage{amsmath,amssymb}
\usepackage{graphicx}
\usepackage{amsthm}
\usepackage{mathtools}
\usepackage{url}
\usepackage{xcolor}
\usepackage{soul}
\usepackage{cancel}
\usepackage{bm}
\usepackage{thmtools, thm-restate}
\usepackage{framed}
\usepackage{authblk}
\usepackage{appendix}

\colorlet{shadecolor}{yellow}

\newtheorem{corol}{Corollary}

\DeclarePairedDelimiter{\floor}{\lfloor}{\rfloor}
\DeclareMathOperator*{\argmax}{arg\,max}
\newcommand{\lemfivet}[1]{\ensuremath{(2^{\floor{\frac{#1}{2}}}-1)}}
\newcommand{\lemfive}[2]{\ensuremath{\lemfivet{#1}+\sqrt{2}\lemfivet{#2}}}

\newcommand{\myover}[2]{\ensuremath{\genfrac{}{}{0pt}{0}{{#1}}{#2}}}
\newcommand{\myreal}{\ensuremath{\mathbb{R}}}
\newcommand{\myvec}[1]{\ensuremath{\bm{#1}}}

\begin{document} 
\title{Deep Radial Kernel Networks: Approximating Radially Symmetric Functions with Deep Networks} 
\author{Brendan McCane}
\author{Lech Szymanski}
\affil{Department of Computer Science\\
University of Otago\\
Dunedin, New Zealand.}

\maketitle

\begin{abstract} 
We prove that a particular deep network architecture is more efficient at approximating radially symmetric functions than the best known 2 or 3 layer networks. We use this architecture to approximate Gaussian kernel SVMs, and subsequently improve upon them with further training. The architecture and initial weights of the Deep Radial Kernel Network are completely specified by the SVM and therefore sidesteps the problem of empirically choosing an appropriate deep network architecture.
\end{abstract} 

\section{Introduction}

Deep networks have been stunningly successful in many machine learning domains since the area was reinvigorated by the work of \citet{krizhevsky2012imagenet}. Despite their success, neural networks in general, have two major limitations:
\begin{enumerate}
\item relatively little is known about them theoretically, although this is changing. In particular, for which problems are deep networks more effective than shallow learners; and
\item deciding on a particular architecture is still an empirical question.
\end{enumerate}

Compare using a deep network for a particular problem with a support vector machine (SVM). The choices for a deep network include: number of layers, number of neurons per layer, activation function, gradient descent algorithm, gradient descent algorithm parameters (e.g. learning rate), which learning tricks to use (dropout, batch normalisation etc), weight initialisation parameters, etc. Getting any one of these wrong can cause the network learning to fail. For an SVM, on the other hand, one needs to choose the kernel function and the kernel function parameters (usually few) and the slack variable parameter $C$. Similarly, other machine learning methods such as decision trees, boosting, random forests are much easier to use, albeit often with worse performance than the best neural network model.

This paper goes some way to addressing both of these limitations and offers a theoretical result and an applied result. The main theoretical result is a new upper bound for the number of neurons required in a deep network to approximate a radially symmetric function. The main applied result uses the theoretical result to construct a deep network approximation of an SVM that can be further trained using back propagation and which often results in improved performance.

\section{Related Work}

It is difficult to deny the empirical success of deep networks. It is still the case that relatively little is known theoretically about their fundamental abilities, although this has been changing over the last few years. Nevertheless, most work consists of existence proofs that do not lead directly to new methods for building or training networks. For example, 
ReLU networks with $n_0$ inputs, $L$ hidden layers of width $n \ge n_0$ can compute functions that have $\Omega\left( (n/n_0)^{(L-1)n_0} n^{n_0} \right)$ linear regions compared to $\sum_{j=0}^{n_0} \binom{n}{j}$ for a shallow network \citep{montufar2014number}. More generally, \citet{telgarsky2016benefits} proved for semi-algebraic neurons (including ReLU, sigmoid etc), that networks exist with $\Theta(k^3)$ layers and $\Theta(k^3)$ nodes that require $\Omega(2^k)$ nodes to approximate with a network of $O(k)$ layers. 
\citet{delalleau2011shallow} show that deep sum-product networks exist for which a shallow network would require exponentially more neurons to simulate. For convolutional arithmetic circuits (similar to sum-product networks), \citet{cohen2016expressive}, in an important result, show that ``besides a negligible (zero measure) set, all functions that can be realized by a deep network of polynomial size, require exponential size in order to be realized, or even approximated, by a shallow network.'' 

The above works, except for \citet{cohen2016expressive} focus on approximating deep networks with shallow networks, but do not indicate what problems are best attacked with deep networks.
For manifolds, \citet{basri2016efficient} show how deep networks can efficiently represent low-dimensional manifolds and that these networks are almost optimal, but they do not discuss limitations of shallow networks on the same problem. Somewhat similarly, \citet{shaham2016provable} show that depth-4 networks can approximate a function on a manifold where the number of neurons depends on the complexity of the function and the dimensionality of the manifold and only weakly on the embedding dimension. Again, they do not discuss the limitations of shallow networks for this problem. Importantly, both of these results are constructive and allow one to actually build the network.

\citet{szymanski2014deep} show that deep networks can approximate periodic functions of period $P$ over $\{0,1\}^N$ with $O(\log_2{N}-\log_2{P})$ parameters versus $O(P \log_2{N})$ for shallow. \citet{eldan2016power} show that 3-layer networks exist such that the network can approximate a radially symmetric function with $O(d^{19/4})$ neurons, that a 2-layer network requires at least $O(e^d)$ neurons.

Therefore evidence is building that deep networks are more powerful than their shallow counterparts in terms of the number of parameters or neurons required. Nevertheless, we shouldn't stop there because there is relatively little work linking this theory with practical applications of the same. In this work we directly extend the work of \citet{szymanski2014deep} and \citet{eldan2016power}. The latter work \citep{eldan2016power} is extended to deeper networks for approximating radially symmetric functions that require fewer parameters than their construction. The former \cite{szymanski2014deep} is extended by generalising their notion of folding transformations to work in multiple dimensions and more simply with ReLU networks. The proofs are constructive and allow us to build networks for approximating radially symmetric functions. These networks are used to approximate Gaussian kernel SVMs and the results show how we can further train these approximations to do better than the original SVM in many cases.

Our applied work is similar in spirit, but quite different in detail to other methods trying to combine deep learning and kernel learning. For example, \citet{wilson2016deep} use a deep network as the kernel in a Gaussian processes framework, but the choice of the deep network architecture remains with the practitioner. Our deep radial kernel could be used as the input kernel in that framework. We also note that multiple kernel learning \citep{lanckriet2004learning} and its descendants take a very different approach to the one we take here --- we make no attempt to learn a kernel matrix, nor to ensure positive semi-definiteness of the resulting network.
 
We start in Section \ref{sec-theory} with the main theoretical results, then proceed to approximating Gaussian kernel SVMs in Section \ref{sec-svm-approx}.

\section{Theory}
\label{sec-theory}

\subsection{Context and Notation}

A radially symmetric function is a function whose value is dependent on the norm of the input only. We are interested in $L$-Lipschitz functions $f$, $|f(x)-f(y)| \le L|x-y|$, as this covers many functions common in classification tasks. The number of dimensions of the input is $d$, and we assume that $f$ is constant outside a radius $R$. This is a similar context to that used by \citet{eldan2016power}. Further, we restrict ourselves to ReLU networks only, which is more restrictive than \citet{eldan2016power}, but allows us to explicitly construct the networks of interest.

Most proofs are only sketched in the main body of the paper. Detailed proofs are provided in the supplementary material.

\subsection{3 Layer Networks}

We start by stating a modified form of Lemma 18 from \citet{eldan2016power}:

\begin{restatable}{lemma}{lemep}
\label{lem18}%
\label{lem-3layer}
Modified form of Lemma 18 from \citet{eldan2016power}:\\
Let $\sigma(z) = \max(0,z)$. Let $f$ be an $L$-Lipschitz function supported on $[0, R]$. Then for any $\delta>0$, there exists a function $g$ expressible by a 3-layer network of width at most $\frac{6 d^2 R^2 + 3 R L}{\delta}$, such that
\begin{equation*}
\sup_{\myvec{x} \in \myreal^d}| g(\myvec{x})-f(||\myvec{x}||)| < \delta + L\sqrt{\delta}.
\end{equation*}

\end{restatable}

The proof follows the basic plan of \citet{eldan2016power} where the first layer is the input layer, the second layer approximates $x_i^2$ for each dimension $i$, and the third layer computes $\sum_i x_i^2$ and approximates $f$.

Since several sections of the second layer are doing the same thing (computing the square of their input), a weight-sharing corollary follows immediately where only one copy of the square approximation is needed.
\begin{corol}[3 Layer Weight Sharing]
Let $\sigma(z) = \max(0,z)$. Let $f$ be an $L$-Lipschitz function supported on $[0, R]$. Then for any $\delta>0$, there exists a function $g$ expressible by a 3-layer weight-sharing network with at most $\frac{6 d R^2 + 3 R L}{\delta}$ weights, such that
\begin{equation*}
\sup_{\myvec{x} \in \myreal^d}| g(\myvec{x})-f(||\myvec{x}||)| < \delta + L\sqrt{\delta}.
\end{equation*}
\end{corol}

\subsection{Deep Folding Networks}
\label{sec-folding}

In this section we show how folding transformations can be used to create a much deeper network with the same error, but many fewer weights than needed in Lemma \ref{lem-3layer}. A folding transformation is one in which half of a space is reflected about a hyperplane, and the other half remains unchanged. Figure \ref{fig-fold-2d} demonstrates how a sequence of folding transformations can transform a circle in 2D to a small sector. After enough folds, we can discard the almost zero coordinates to approximate the norm.
We will use this general idea to prove the following theorem:
\begin{restatable}{theorem}{main}
\label{thm:main}%
Let $\myvec{x} \in \myreal^d$, and $\sigma(z) = \max(0,z)$. Let $f$ be an $L$-Lipschitz function supported on $[0,R]$. Fix $L,\delta,R > 0$. There exists a function $g$ expressible by a $O(d\log_2(d) + \log_2(d)\log_2\left(\frac{R}{\sqrt{\delta}}\right))$ layer network where the number of weights, and number of neurons, $N_w, N_n = O(d^2 + d\log_2\left(\frac{R}{\sqrt{\delta}}\right) + \frac{3 RL}{\delta})$ , such that:
\begin{equation*}
\sup_{\myvec{x} \in \myreal^d} | g(\myvec{x}) - f(||\myvec{x}||)| < L\sqrt{\delta} + \delta
\end{equation*}
\end{restatable}

The approach taken here is a constructive one and specifies the architecture of the network needed to approximate $f$. In fact, all of the weights except those in the last layer are specified. The approach is somewhat different to that used to prove Lemma \ref{lem-3layer}. We build a sequence of layers to directly approximate $||\myvec{x}||$ and then approximate $f$ in the last layer. To build our layers, we need a few helper lemmas.

\begin{figure}
\centering
\includegraphics[width=0.4\textwidth]{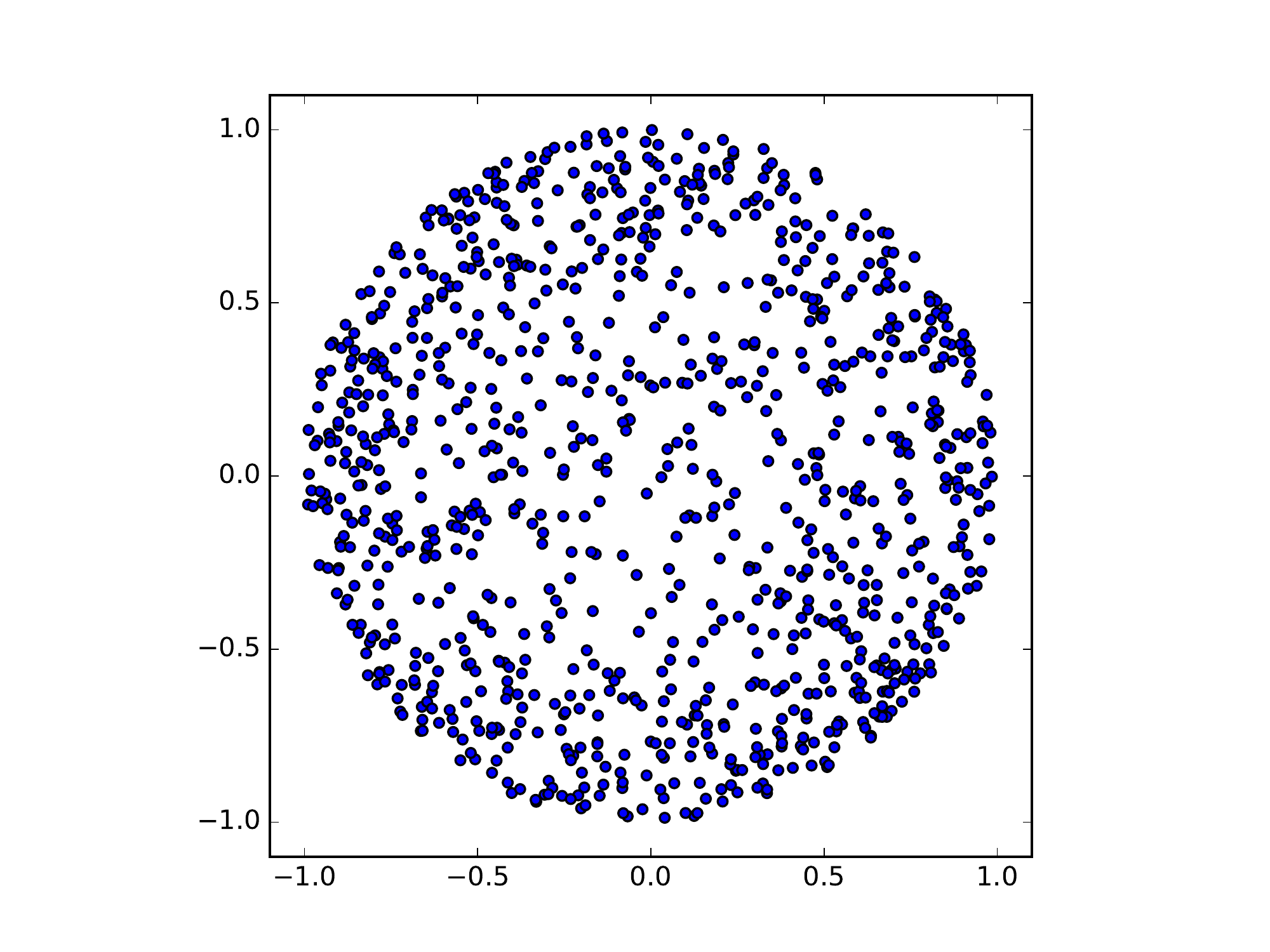}
\includegraphics[width=0.4\textwidth]{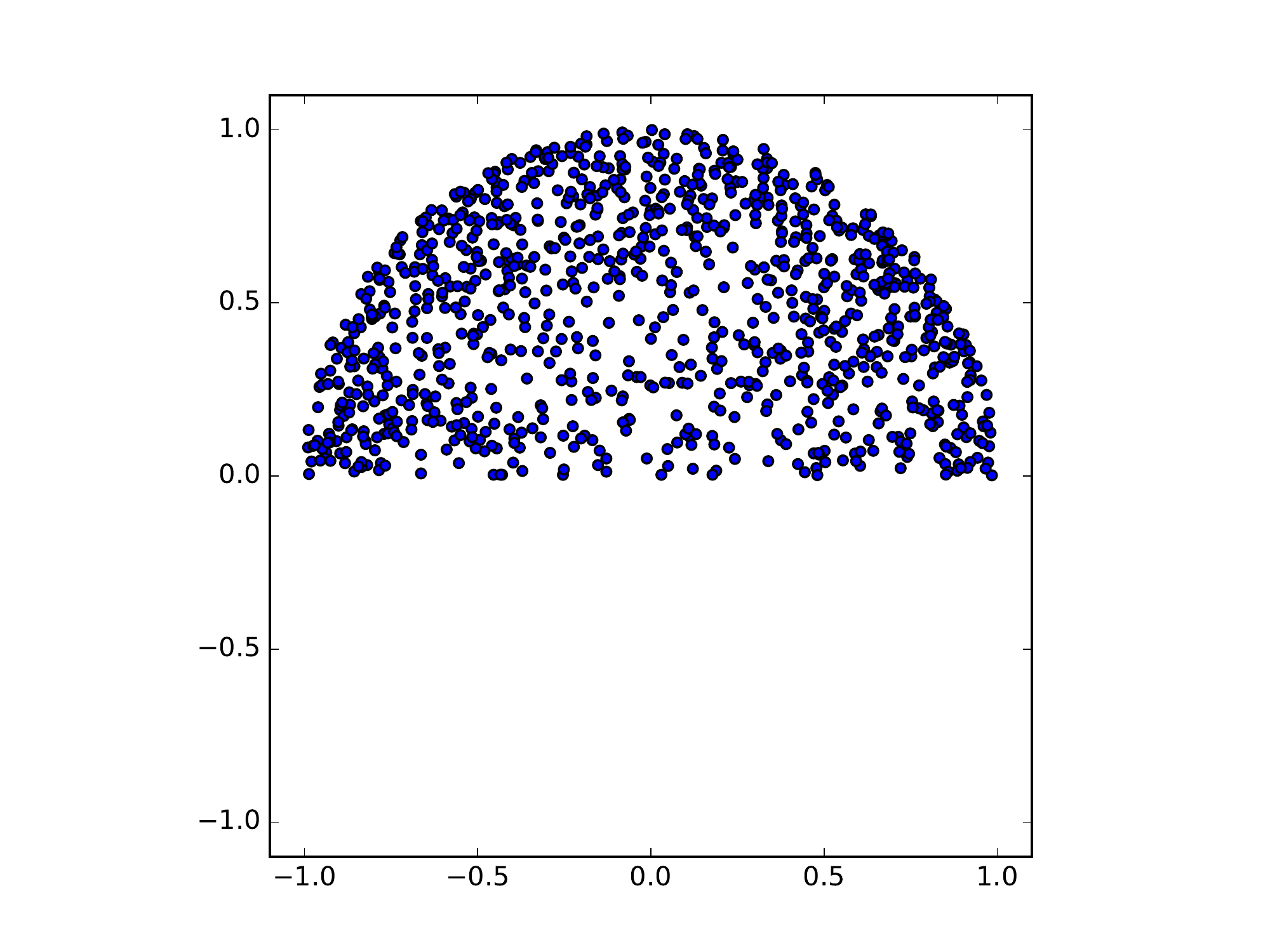}\\
\includegraphics[width=0.4\textwidth]{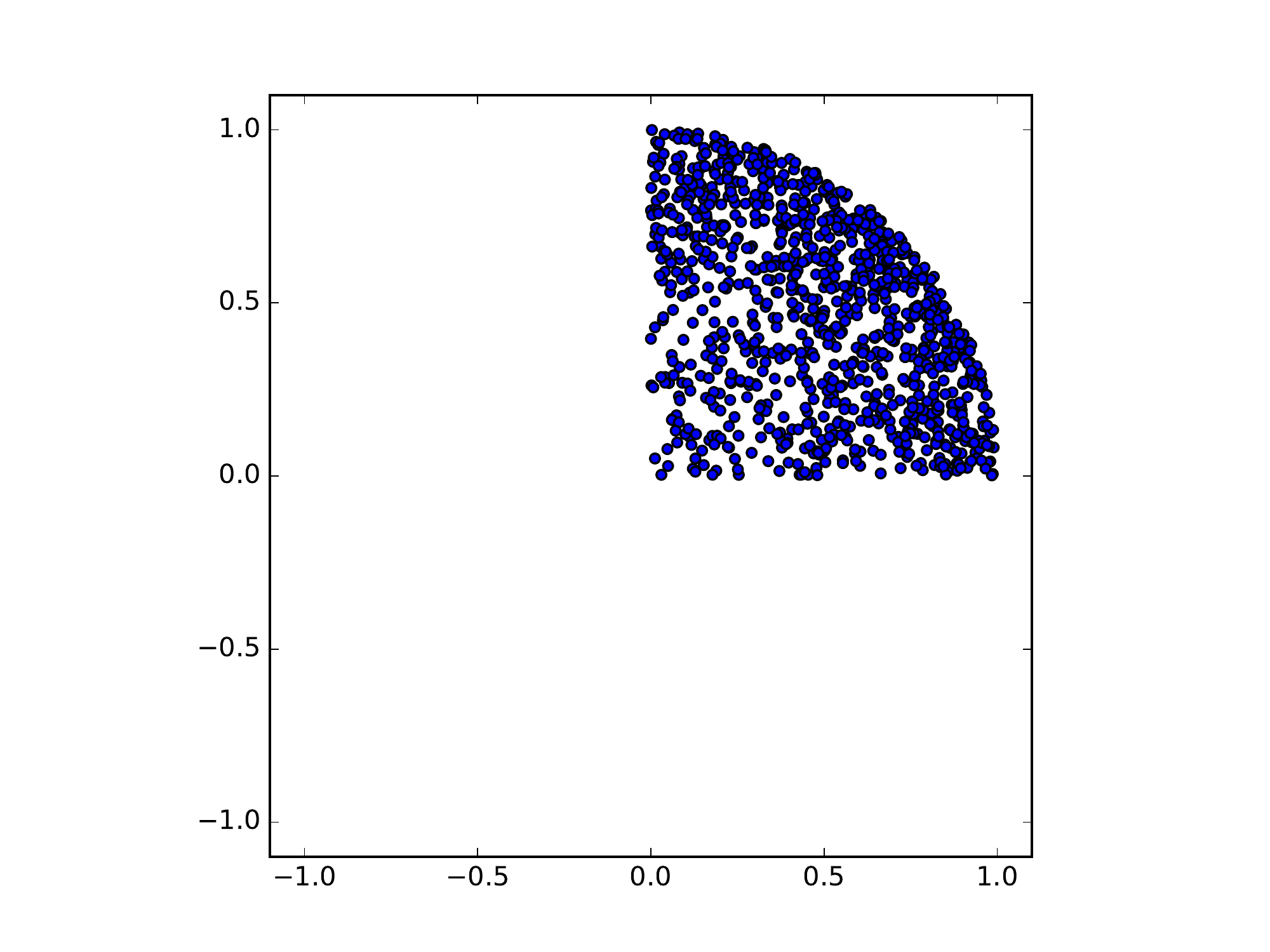}
\includegraphics[width=0.4\textwidth]{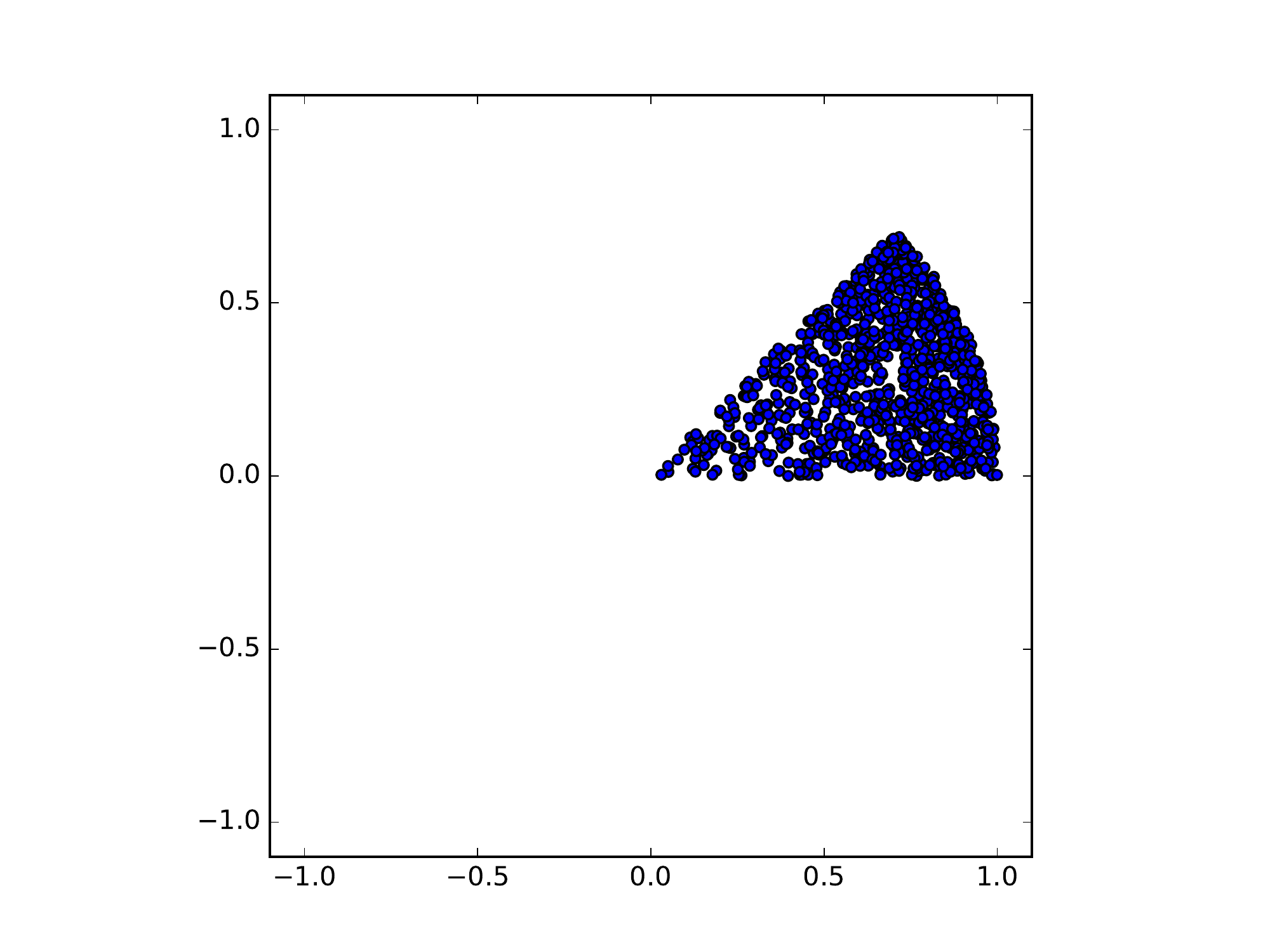}\\
\includegraphics[width=0.4\textwidth]{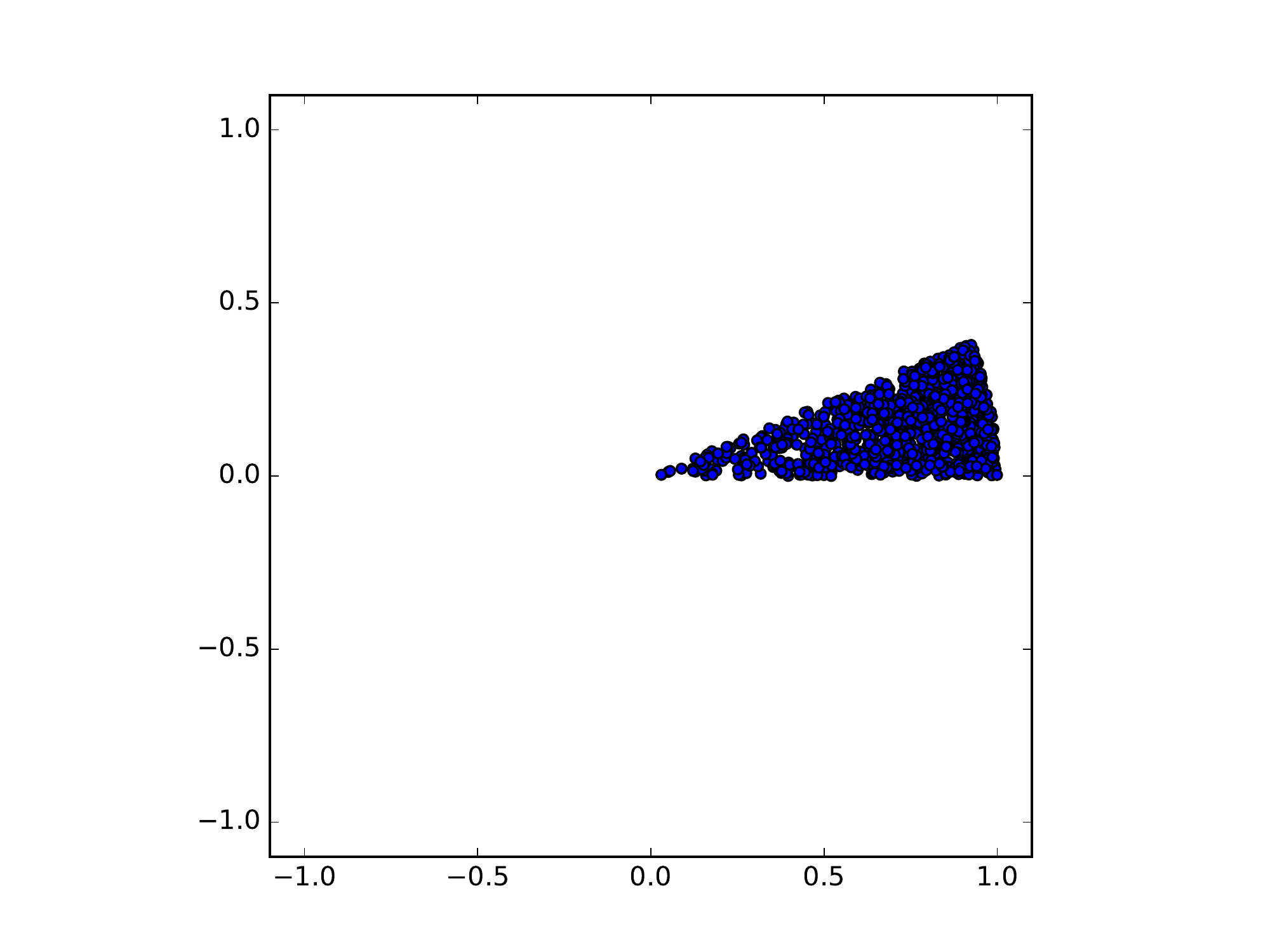}
\includegraphics[width=0.4\textwidth]{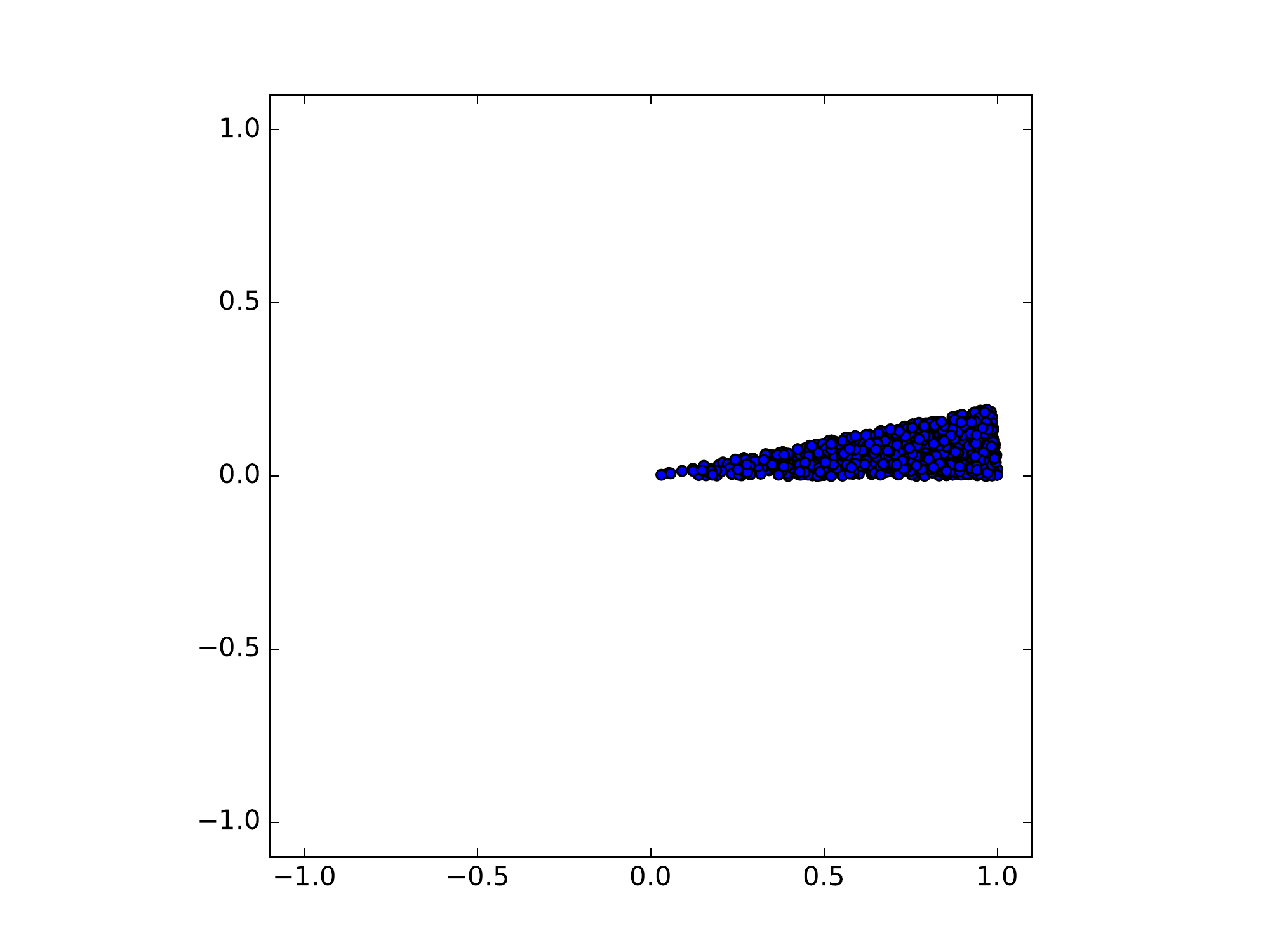}
\caption{Series of folding transformations for 2D. \label{fig-fold-2d}}
\end{figure}

\begin{restatable}[2D fold]{lemma}{lemtwodfold}
There exists a function $g: \myreal^2 \to \myreal^2$, expressible by a ReLU network with 4 ReLU units and 2 sum units that can compute a folding transformation about a line through the origin, represented by the unit direction vector $\myvec{l}=(l_x, l_y)^T$. The function $g$ is of the form:
\begin{equation*}
g(\myvec{x}) = \begin{cases}
\myvec{x} & \myvec{l} \cdot \myvec{x^\perp} > 0 \\
\begin{bmatrix}
l_x^2-l_y^2 & 2 l_x l_y \\
2 l_x l_y & l_y^2-l_x^2
\end{bmatrix}
\myvec{x} & \text{otherwise.}
\end{cases}
\end{equation*}
\label{lem-2d-fold}
\end{restatable}

The requisite ReLU network is shown in Figure \ref{fig-relu-fold-2d}. Only one of the nodes labeled $x_{-}$ ($y_{-}$) and $x_{+}$ ($y_{+}$) are active at any one time. Therefore there are four possible cases depending on which two nodes are active. Note that $x_{-}$ is active when $\myvec{l} \cdot \myvec{x^\perp} < 0$ and $x_{+}$ is active when $\myvec{l} \cdot \myvec{x^\perp} > 0$.
To approximate the 2D norm, we simply stack layers of the type shown in Figure \ref{fig-relu-fold-2d} with suitable choice of $l_x, l_y$ at each layer. Note that the summation nodes aren't required since they can be incorporated into the summations and weights of the next ReLU layer. 
These 2D folds can be used to estimate the norm of a vector as per the following lemma.

\begin{figure}
\centering
\includegraphics[width=0.9\textwidth]{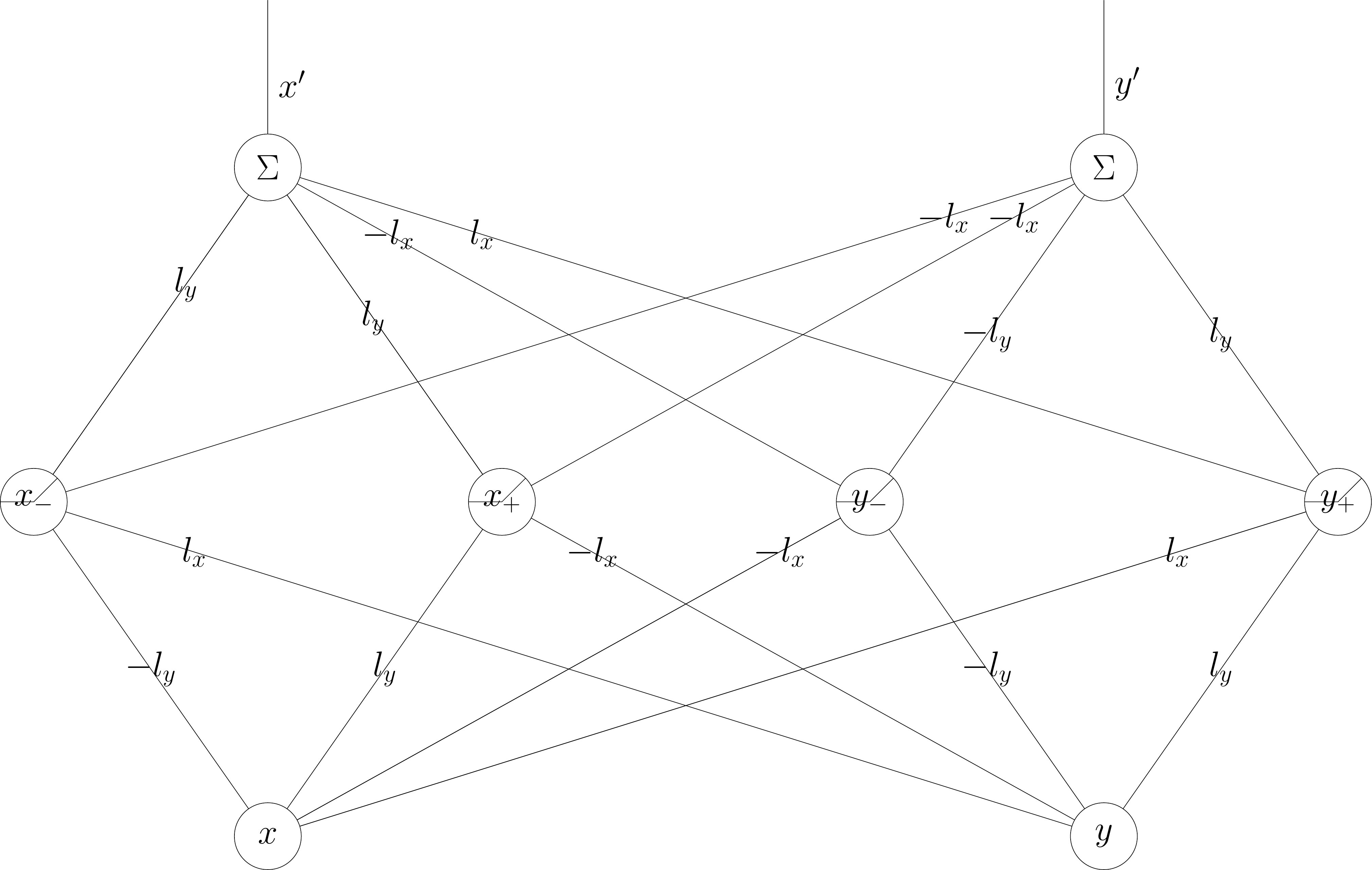}
\caption{A network to produce a 2D fold. \label{fig-relu-fold-2d}}
\end{figure}

\begin{restatable}[Approximate $||\myvec{x}||, x \in \myreal^2, ||x||<R$]{lemma}{lemapproxtwod}
There exists a function $g$, expressible by a ReLU network with no more than $\log_2\left(R \frac{\pi}{\delta}\right)$ layers and $4$ nodes per layer such that:
\begin{equation*}
\sup_{\myvec{x}\in \myreal^2, ||\myvec{x}||\le R} | g(\myvec{x}) - ||\myvec{x}|| | \le \delta
\end{equation*}
\label{lem-approx-2d}
\end{restatable}
\begin{proof}
The proof is short and simple. After $f$ layers, each data point will be within an angle of $\frac{\pi}{2^{f}}$ of the $x$-axis. Simple geometry and appropriate approximations leads to:
\begin{align*}
\delta &=||\myvec{x}||-||\myvec{x}||\cos\left(\frac{\pi}{2^{f}}\right)\\
&\le R\left(1-\cos\left(\frac{\pi}{2^{f}}\right)\right)\\
&\le R\left(2 \sin\left(\frac{\pi}{2^{f+1}}\right)\right)\\
&\le R\left(\frac{\pi}{2^{f}}\right)\\
f &\le \log_2\left(R \frac{\pi}{\delta}\right)
\end{align*}
\end{proof}

The following lemma generalises this construction to folds in $d$ dimensions.
\begin{restatable}[Approximate $||\myvec{x}||, x \in \myreal^d$]{lemma}{lemapproxvecx}
\label{lem-approx-vecx}
There exists a function $g$, expressible by a ReLU network with:
\begin{align*}
\small
N_l \le &\log_2(d)\\
&\log_2\left(\frac{R \pi}{\delta} \left[\lemfive{(d+1)}{d} \right]\right) \\
N_n \le &4(d-1) \\
&\log_2\left(\frac{R \pi}{\delta} \left[\lemfive{(d+1)}{d} \right]\right) \\
N_w \le& 8(d-1)\\
&\log_2\left(\frac{R \pi}{\delta} \left[\lemfive{(d+1)}{d} \right]\right)
\end{align*}
such that:
\begin{equation*}
\sup_{\myvec{x}\in \myreal^d, ||\myvec{x}||\le R} | g(\myvec{x}) - ||\myvec{x}|| | \le \delta
\end{equation*}
\end{restatable}
We note that a fold in a 2D plane in $\myreal^d$ will leave all coordinates perpendicular to the plane unchanged. We can therefore apply the approximation of Lemma \ref{lem-approx-2d} to pairs of input coordinates to produce $d/2$ new coordinates. Then apply the same reduction to produce $d/4$ coordinates and continue on this way until there is only one coordinate left. In effect, we are calculating the norm via the following scheme:
{\small
\begin{align*}
&\sqrt{x_1^2 + x_2^2 + x_3^2 + \cdots + x_d^2} =\\
&\sqrt{\sqrt{ \sqrt{x_1^2 + x_2^2}^2 + \sqrt{x_3^2+x_4^2}^2}^2 \cdots \sqrt{\cdots + \sqrt{x_{n-1}^2 + x_{n-2}^2}^2}^2}
\end{align*}
}
Figure \ref{fig-multiple-folds} shows the resulting network. The proof is by induction and is rather long so is not produced here but appears in full in the supplementary material.

At this point we make use of Lemma 19 from \citet{eldan2016power} which we reproduce here:
\begin{restatable}[Lemma 19 from \citet{eldan2016power}]{lemma}{lemnineteen}
\label{lem19}
Let $\sigma(z) = \max(0,z)$ be the ReLU activation function, and fix $L,\delta,R>0$. Let $f: \myreal \to \myreal$ which is constant outside an interval $[-R,R]$. There exist scalars $a$, $\{\alpha_i,\beta_i\}_{i=1}^w$, where $w \le 3\frac{RL}{\delta}$, such that the function:
\begin{equation}
h(x) = a + \sum_{i=1}^w \alpha_i \sigma(x-\beta_i)
\end{equation}
is L-Lipschitz and satisfies:
\begin{equation}
\myover{\sup}{x \in \myreal} | h(x) - f(x) | \le \delta.
\end{equation}
Moreover, one has $|\alpha_i| \le 2L$ and $w \le 3 \frac{RL}{\delta}$.
\end{restatable}

We can now prove the main theorem (some steps left out for brevity):
\main*
\begin{proof}
From Lemma \ref{lem-approx-vecx} we can approximate $||\myvec{x}||$ to within $\sqrt{\delta}$ and using Lemma \ref{lem19}:
\begin{equation*}
f(||\myvec{x}||)+L\sqrt{\delta}-\delta \le g(||\myvec{x}||+\sqrt{\delta}) \le f(||\myvec{x}||)+L\sqrt{\delta}+\delta
\end{equation*}
and 
\begin{equation*}
f(||\myvec{x}||)-L\sqrt{\delta}-\delta \le g(||\myvec{x}||+\sqrt{\delta}) \le f(||\myvec{x}||)-L\sqrt{\delta}+\delta
\end{equation*}
therefore:
\begin{equation*}
f(||\myvec{x}||)-L\sqrt{\delta}-\delta \le g(||\myvec{x}||+\sqrt{\delta}) \le f(||\myvec{x}||)+L\sqrt{\delta}+\delta
\end{equation*}

\noindent The number of weights and neurons required by Lemma \ref{lem19} is $3 \frac{RL}{\delta}$. The number of weights and neurons required to estimate $||\myvec{x}||$ is given by Lemma \ref{lem-approx-vecx} (substituting $\sqrt{\delta}$ for $\delta$). Stack the network from Lemma \ref{lem19} (1 layer) onto the end of the network from Lemma \ref{lem-approx-vecx} ($O\left(d\log_2(d) + \log_2(d)\log_2\left(R/\sqrt{\delta}\right)\right)$ layers), thus requiring a total number of neurons no more than:
\begin{align*}
N_n \le \bigg[ &4(d-1) \\
& \log_2{\left(\frac{R\pi}{\sqrt{\delta}} \left[ \lemfive{(d-1)}{d} \right]\right)} \bigg]\\
+ &\frac{3RL}{\delta} \\
N_n &= O\left(d^2 + d\log_2\left(\frac{R}{\sqrt{\delta}}\right) + \frac{3 RL}{\delta}\right)
\end{align*}
and a total number of weights no more than:
\begin{align*}
N_w \le \bigg[  &8(d-1) \\
& \log_2{\left(\frac{R\pi}{\sqrt{\delta}} \left[ \lemfive{(d-1)}{d} \right]\right)}
\bigg] \\
+ &\frac{3RL}{\delta}\\
N_w &= O\left(d^2 + d \log_2\left(\frac{R}{\sqrt{\delta}}\right) + \frac{3 RL}{\delta}\right)
\end{align*}
\end{proof}

Again, there is an obvious weight-sharing corollary:

\begin{corol}[Deep weight sharing network]
Let $\myvec{x} \in \myreal^d$, and $\sigma(z) = \max(0,z)$. Let $f$ be an $L$-Lipschitz function supported on $[0,R]$. Fix $L,\delta,R > 0$. There exists a function $g$ expressible by a network where the number of weights is at most  $N_w = O\left(d + \log_2\left(\frac{R}{\sqrt{\delta}}\right) + \frac{3 RL}{\delta}\right)$ , such that:
\begin{equation*}
\sup_{\myvec{x} \in \myreal^d} | g(\myvec{x}) - f(||\myvec{x}||)| < L\sqrt{\delta} + \delta
\end{equation*}
\end{corol}

Comparing Theorem \ref{thm:main} to Lemma \ref{lem18}, both are of order $d^2$, however the folding network version is more efficient in terms of $R$ and $\delta$. This means the deeper network can be much more efficient when either $d$ or $R$ is large, or $\delta$ is small.

\section{Deep Radial Kernel Network (DRKN)}
\label{sec-svm-approx}

We have shown constructively that a deep network using folding transformations can more efficiently represent finite extent radially symmetric functions than a corresponding 3-layer network. This construction can therefore be used to approximate any system that makes use of radial functions, and could be particularly useful for approximating radially symmetric Gaussians. There are advantages and disadvantages to doing so. The main advantage is that it allows the power and flexibility of deep neural networks to be applied in a systematic way with the architecture specified by the problem at hand. The disadvantage is that it can be more computationally costly to evaluate the radial functions via a deep network compared to directly using the function itself. However, this cost is offset by the extra flexibility afforded by the deep structure. After initialisation, the network can be further trained, allowing it to adapt more towards the data and away from the radial function approximation. In this section we demonstrate this idea on support vector machines with Gaussian kernels and show empirically that such a construction tends to perform better than the corresponding SVM but does not appear to suffer greatly from overtraining.

The first step in creating a DRKN is to train a support vector machine. The method will work for any SVM (or support vector regression) that uses a radially symmetric kernel. The most popular is the Gaussian kernel and that is what we use here.
In this case, for multi-class problems, we use one-vs-rest SVMs. The decision function of the SVM is:
\begin{equation}
f(\myvec{X}) = \argmax_{1\le c \le N_C} \sum_{i=1}^{N_{C,V}} \alpha_{c,i} K(V_{c,i}, X),
\label{eqn-svm-decision}
\end{equation}
where $N_C$ is the number of classes, $N_{C,V}$ is the number of support vectors for class $C$, $\alpha_{c,i}$ is the coefficient for support vector $V_{c,i}$, and $K(\cdot,\cdot)$ is the kernel function. Note that $\alpha_{c,i}$ can be positive or negative as it incorporates the class label of the relevant binary classification problem (1 for the class of interest, and -1 for all other classes).

An SVM with a Gaussian kernel has two parameters: $\sigma$ which specifies the width of the kernel; and $C$, the trade-off between misclassification and decision surface smoothness. These parameters need to be estimated or specified to train an SVM. We make use of the Python scikit learn package \citep{scikit-learn} and a randomised search over an exponential distribution to estimate the optimal parameters for the SVM.

There are several ways to convert Equation \ref{eqn-svm-decision} into a deep network using the techniques of Section \ref{sec-folding}, but for this paper we use the most direct method. The majority of the network relates to approximating the kernel. This kernel is weight shared across all support vectors as in the support vector machine (one alternative is to have a different kernel for each support vector). The kernel used for training the SVM is a Gaussian kernel, however Section \ref{sec-folding} requires a kernel of finite support. For the DRKN we approximate the Gaussian kernel using the polynomial method of \cite{fornefett2001radial}[$Q_{3,1}$] first, then approximate the polynomial using the method of Section \ref{sec-folding}. The network implements the following decision function:

{\small
\begin{equation}
f'(\myvec{X}) = \argmax_{1\le c \le N_C} \left[ \frac{1}{2} + \frac{1}{2} \tanh\left(\sum_{i=1}^{N_{C,V}} \alpha_{c,i} F_n(V_{c,i}-X) \right) \right],
\label{eqn-fold-decision}
\end{equation}}

where $F_n$ is the fold network approximation. We use the cross-entropy softmax loss function and optimise over all weights in the fold network, the support vector centres, and the support vector weights. A conceptual diagram of part of the network for one class is given in Figure \ref{fig-deep-svm}.\footnote{Code for approximating an SVM and training a DRKN can be downloaded from \url{https://bitbucket.org/mccane/deep-radial-kernel-network}.} The cross-entropy objective function is used with stochastic gradient descent. The number of samples in each mini-batch varies depending on the problem, and ranges from 10 to 100.

\begin{figure*}
\centering
\includegraphics[width=0.9\textwidth]{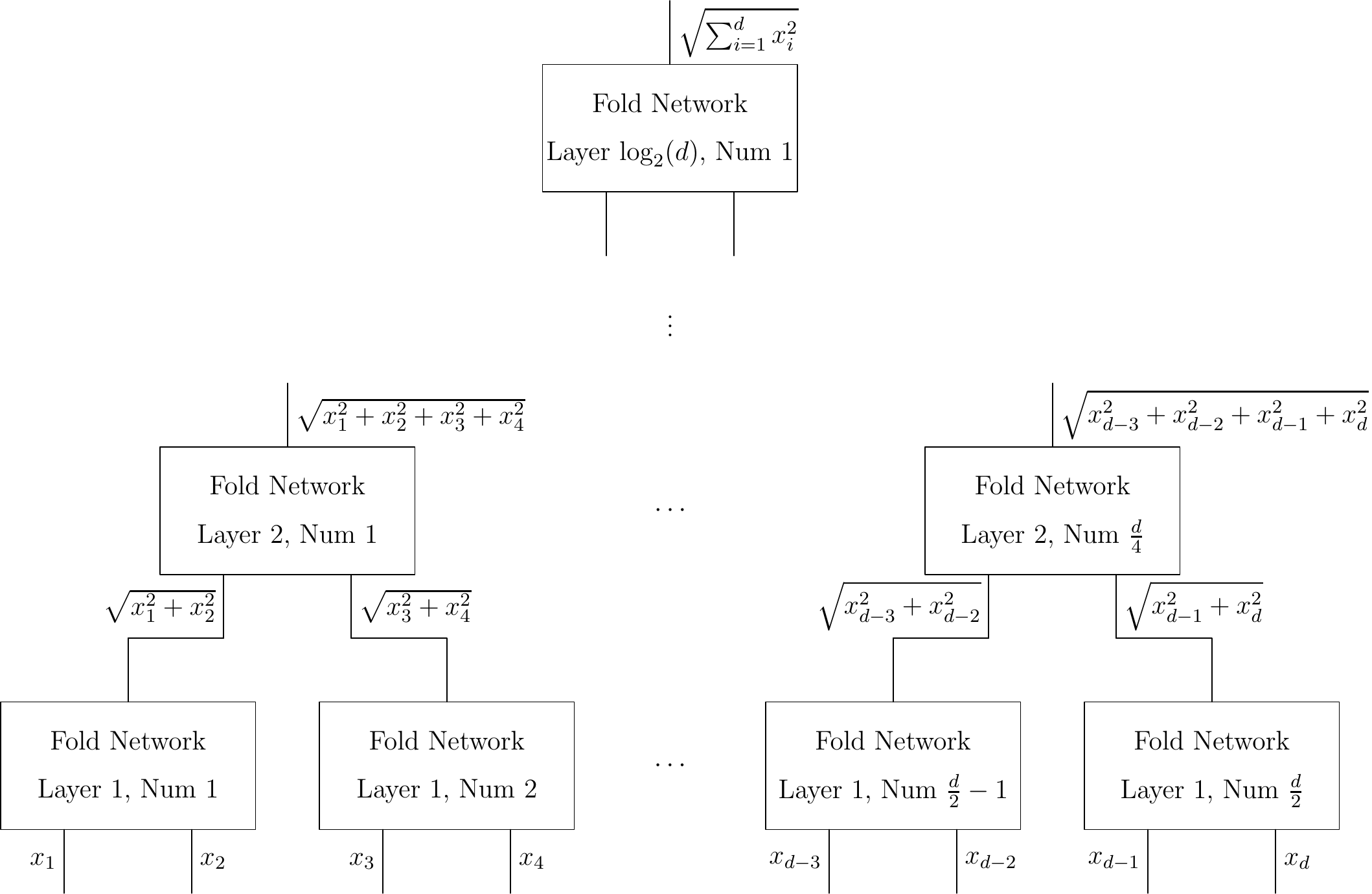}
\caption{A network to approximate the norm of a vector. Each fold network consists of multiple layers. \label{fig-multiple-folds}} 
\end{figure*}
\begin{figure*}
\centering
\includegraphics[width=0.9\textwidth]{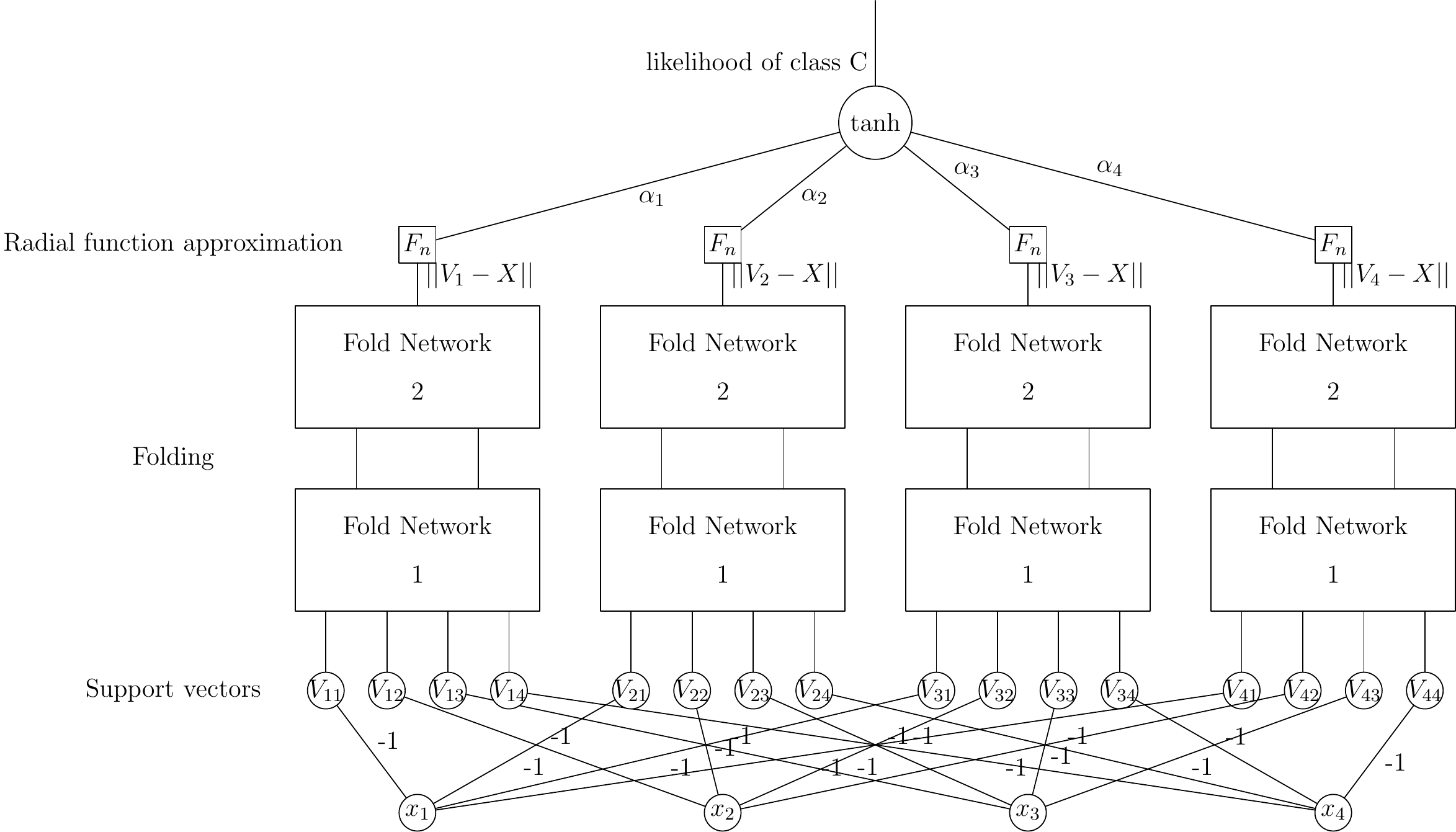}
\caption{Part of a Deep Radial Kernel Network showing the sub-network for a single class for a 4-dimensional problem. There are also (coincidentally) 4 support vectors. Each fold network with the same number has identical weights as does the radial function approximation network $F_n$. Circle nodes indicate single units, while rectangular nodes indicate sub-networks consisting of potentially many neurons. The support vector units effectively behave as single input linear units with a bias (the support vector component). \label{fig-deep-svm}}
\end{figure*}

\subsection{Datasets}

Several standard datasets have been used to test the algorithm and these are listed in Table \ref{tab-datasets}. Most datasets were sourced from the UCI machine learning repository, with the covtype dataset coming via the Python module sklearn. If the dataset was already split into train and test sets, then our testing made use of these sets. If not, then the training set was split 70/30 into train and test sets except for the covtype dataset. In that case, there were too many samples for effective training of an SVM, so 100000 data points were randomly sampled from the set, and these were split 70/30 into train and test sets. In all cases, the SVM and deep network were trained and tested on the same data. 

\begin{table}
\centering
\begin{tabular}{|c|c|c|c|c|}
\hline
Name&Source&Dims&Training \#&Test \#\\
\hline
svmguide1&UCI&4&3090&400\\
whitewine&UCI&11&3428&1470\\
redwine&UCI&8&1112&480\\
breastcancer&UCI&10&490&210\\
sat&UCI&37&4435&2000\\
sensorless drive&UCI&48&46808&11702\\
segmentation&UCI&20&1617&694\\
covtype&sklearn&54&70000&30000\\
\hline
\end{tabular}
\caption{Datasets used to test the algorithm and compare with SVM. \label{tab-datasets}}
\end{table}

\subsection{Results}

Figure \ref{fig-results} shows the results for the 8 example problems with train and test error shown as the number of epochs increases. For reference we include the SVM test error and training and test error for a radial basis function (RBF) network that was initialised with the same kernel and support vectors as the SVM - essentially replacing the fold network and the function approximate in Figure \ref{fig-deep-svm} with a Gaussian RBF neuron. For the RBF network the Gaussian parameter, the support vectors and all the network weights are trainable. 
The RBF network was included to test whether moving the support vectors and/or adjusting the width of the Gaussians were the factors producing improvement.

Note that there is a training anomaly in the covtype results for the DRKN. We think this is due the optimisation algorithm taking a misstep but quickly recovering. In any case, it has little effect on the long term results.

\begin{figure*}
\centering
\includegraphics[width=0.4\textwidth]{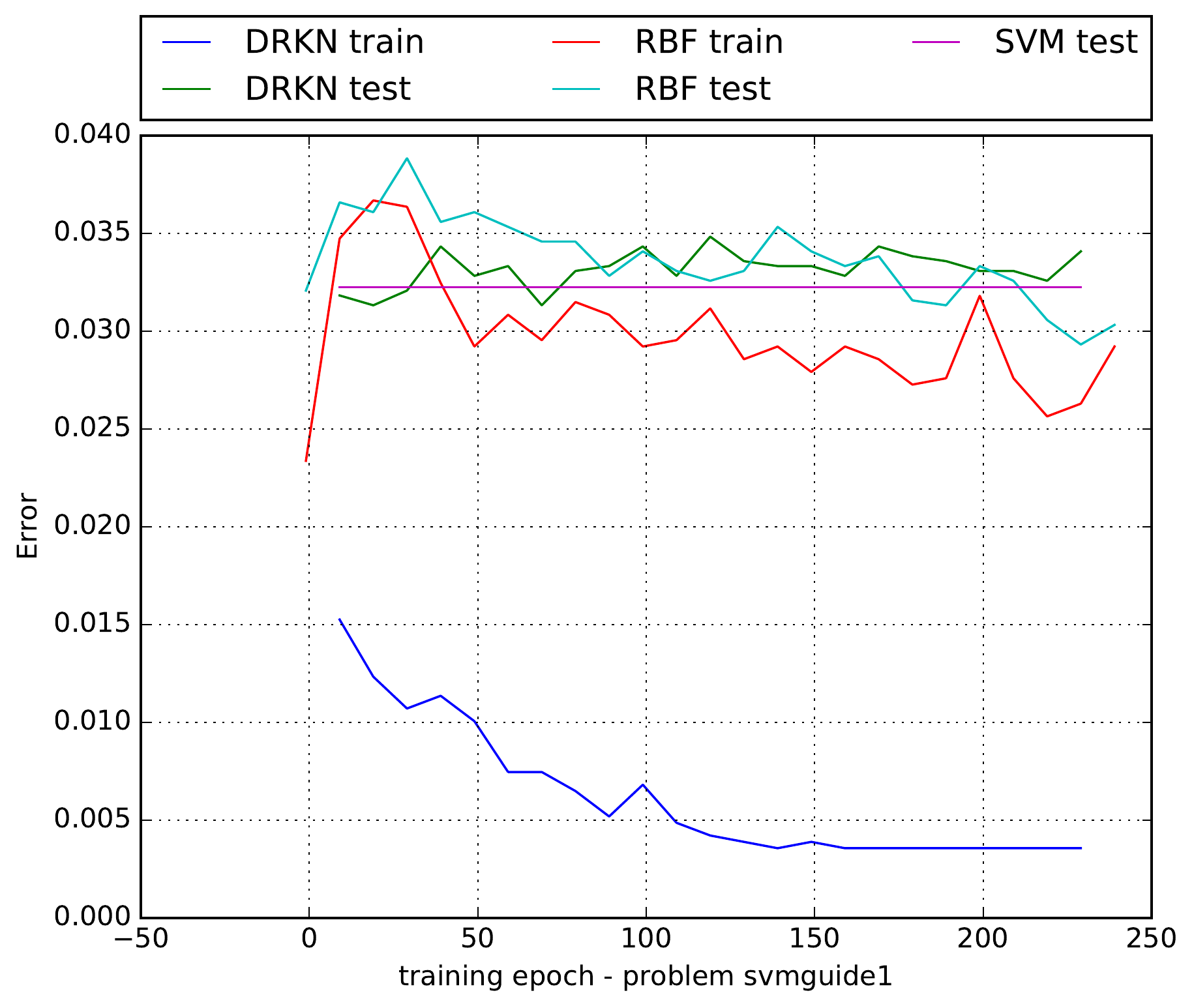}
\includegraphics[width=0.4\textwidth]{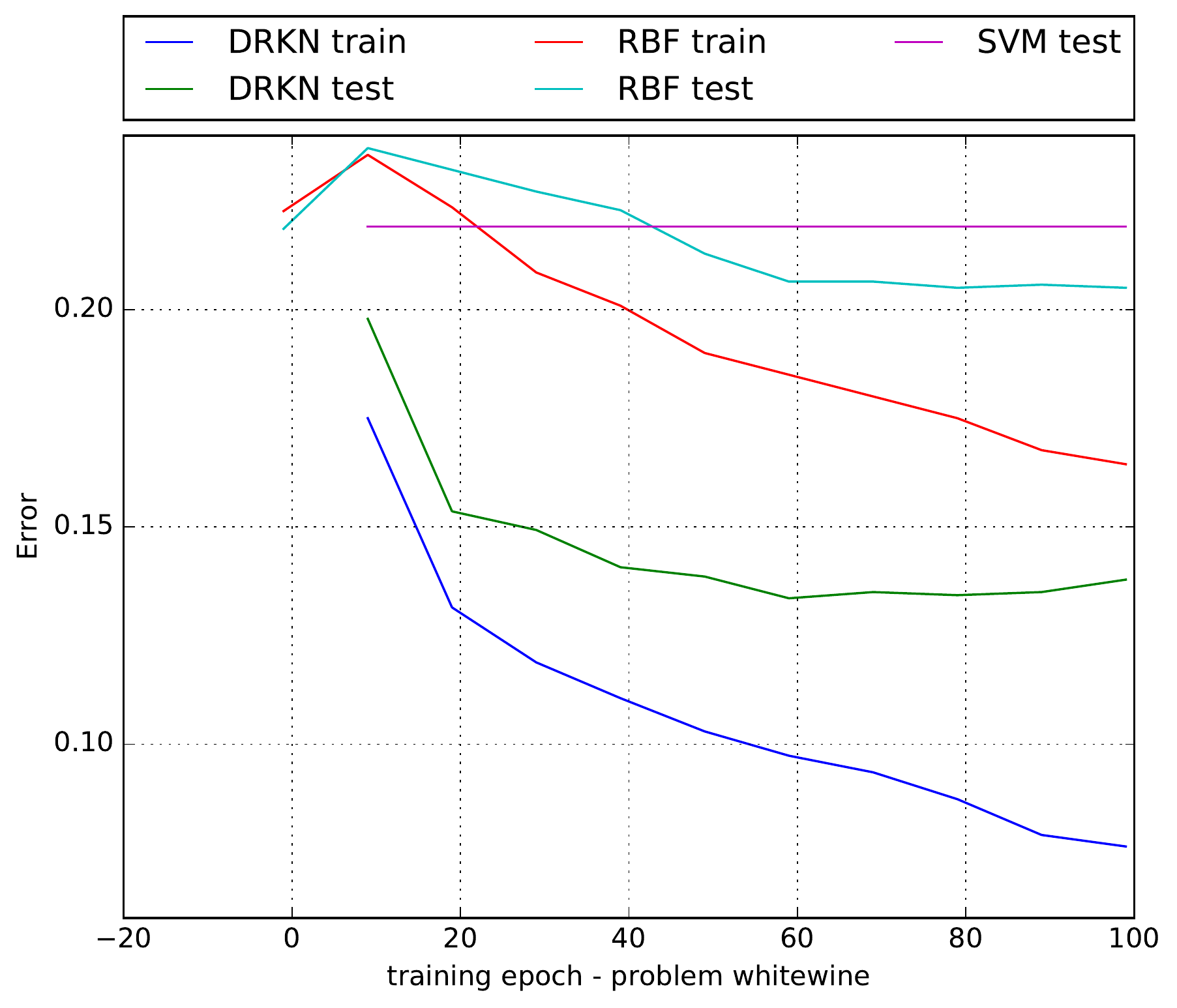}
\includegraphics[width=0.4\textwidth]{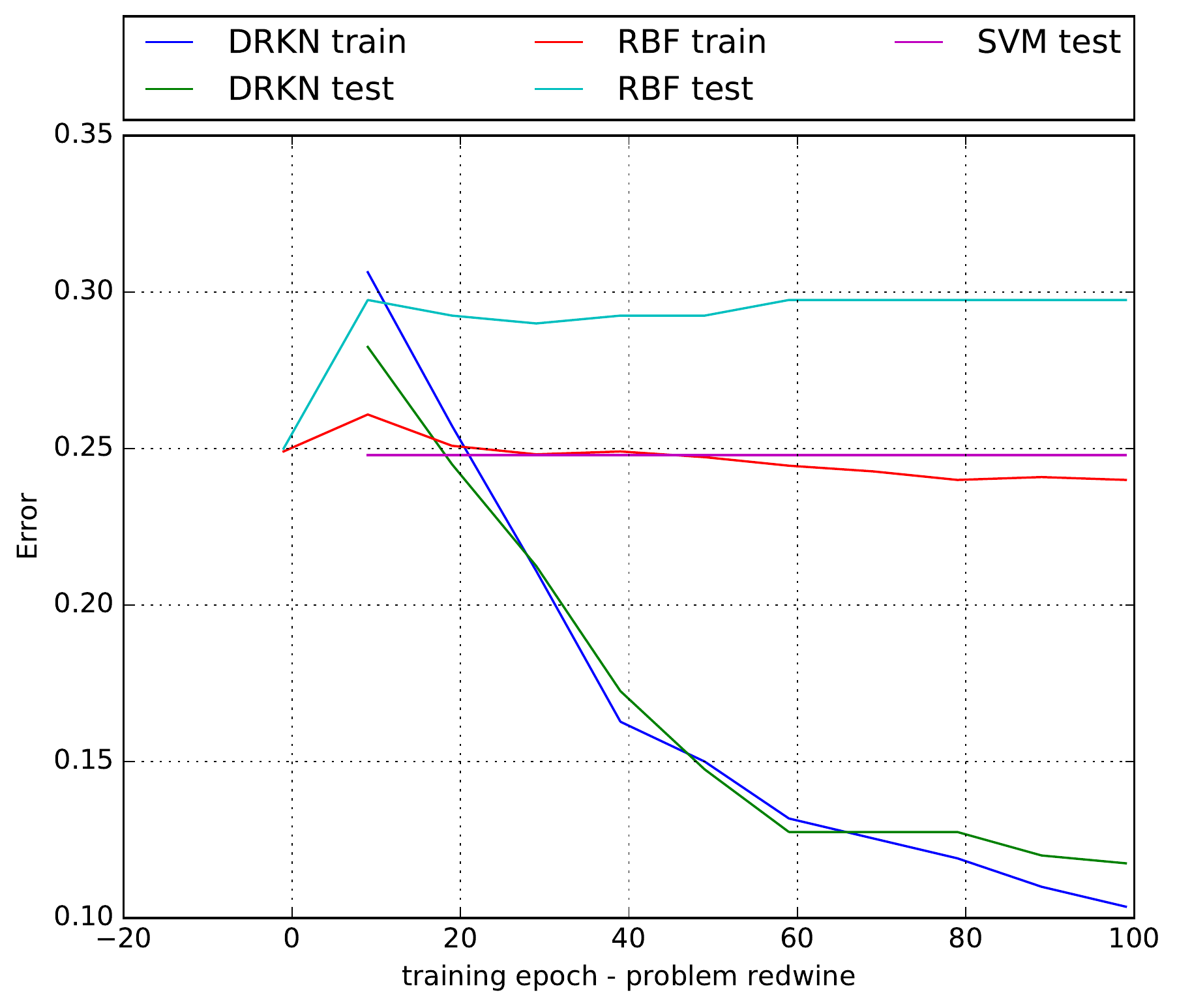}
\includegraphics[width=0.4\textwidth]{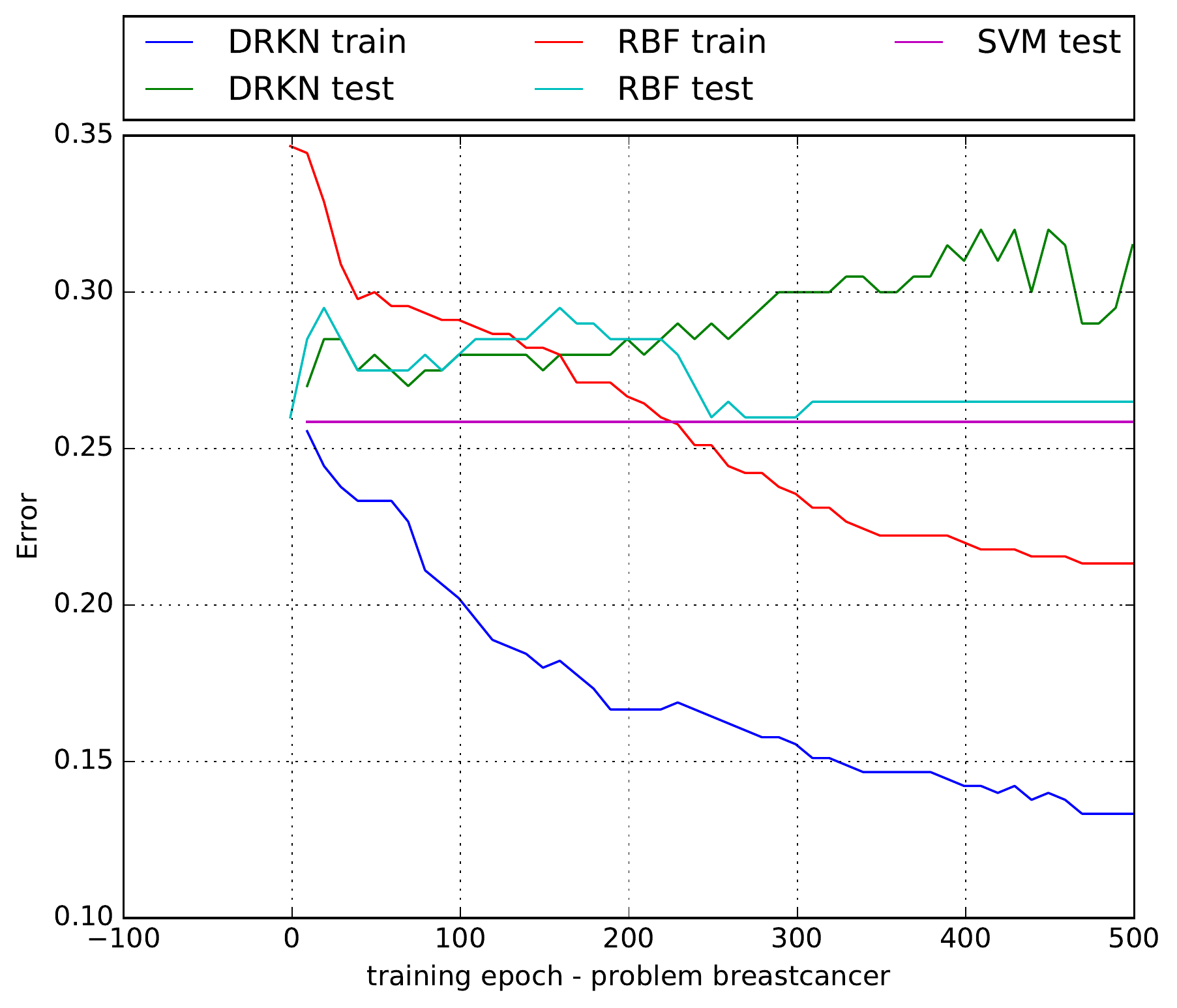}
\includegraphics[width=0.4\textwidth]{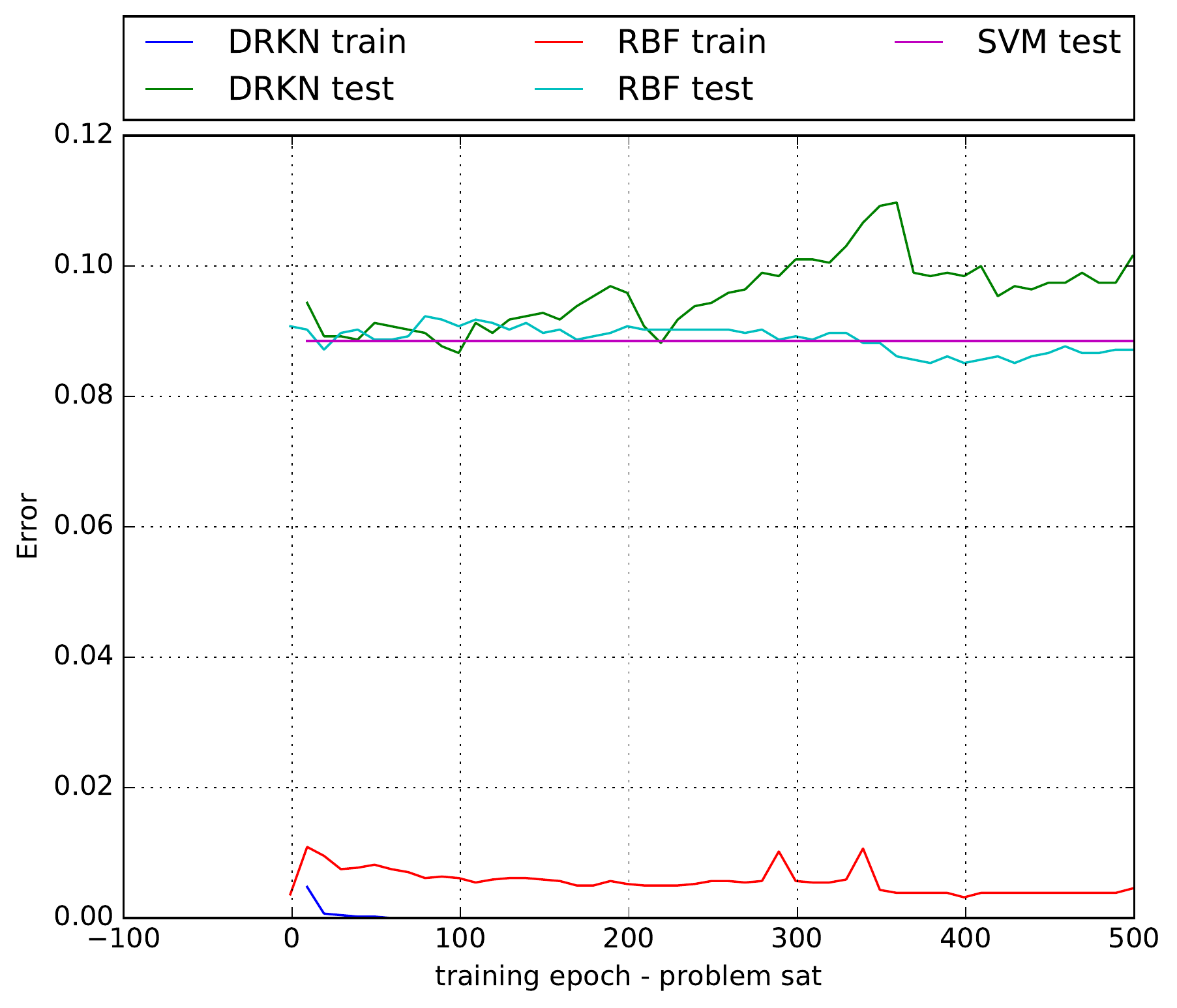}
\includegraphics[width=0.4\textwidth]{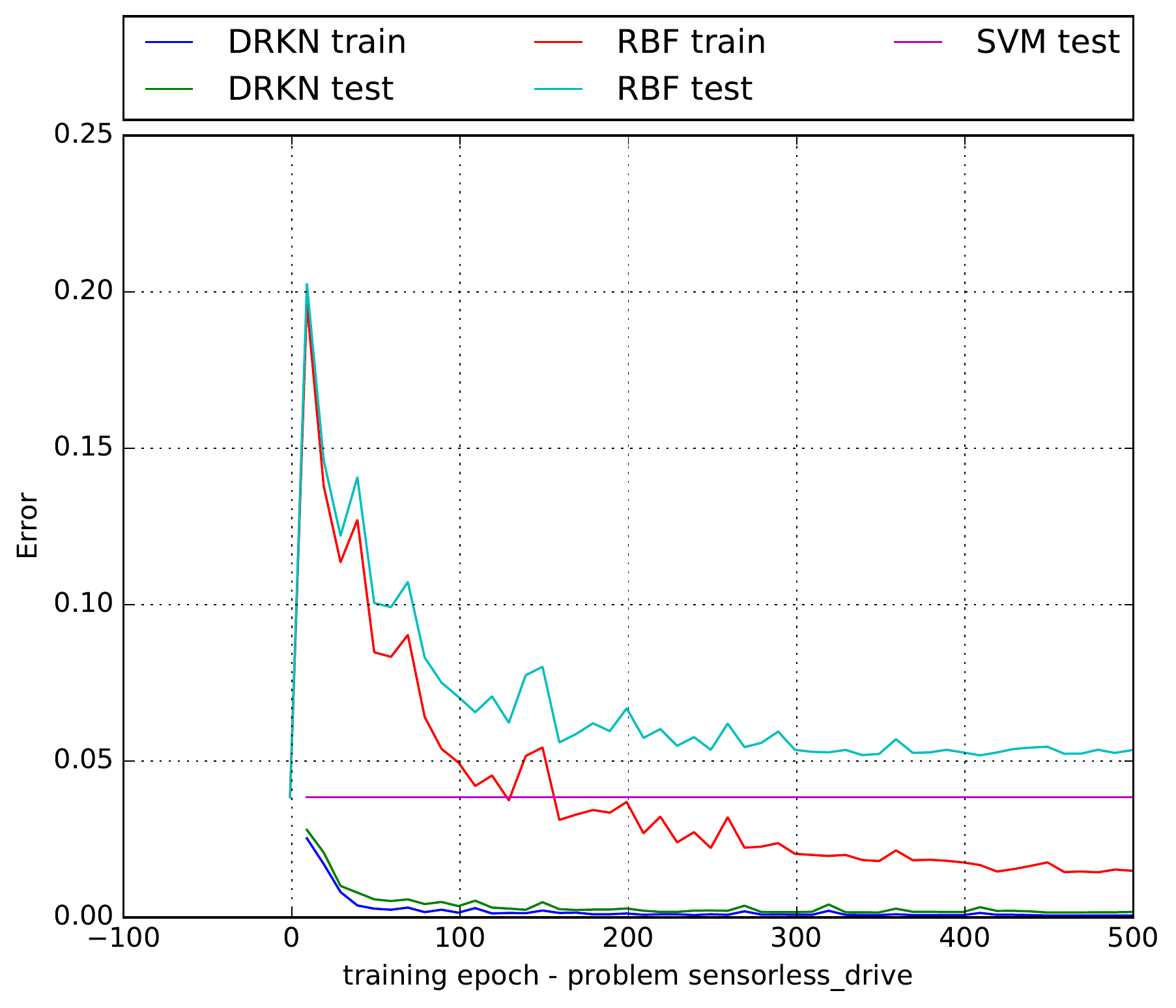}
\includegraphics[width=0.4\textwidth]{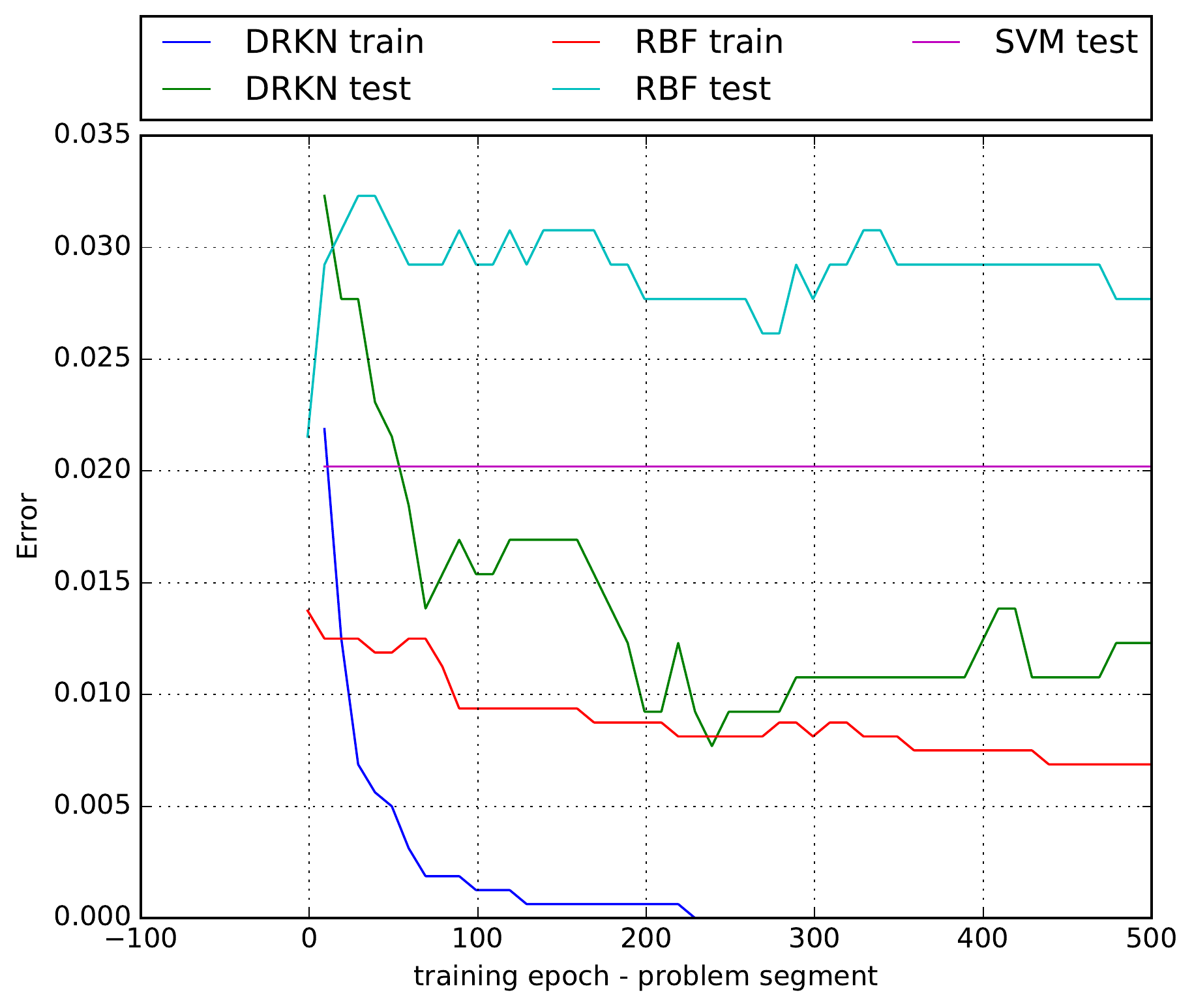}
\includegraphics[width=0.4\textwidth]{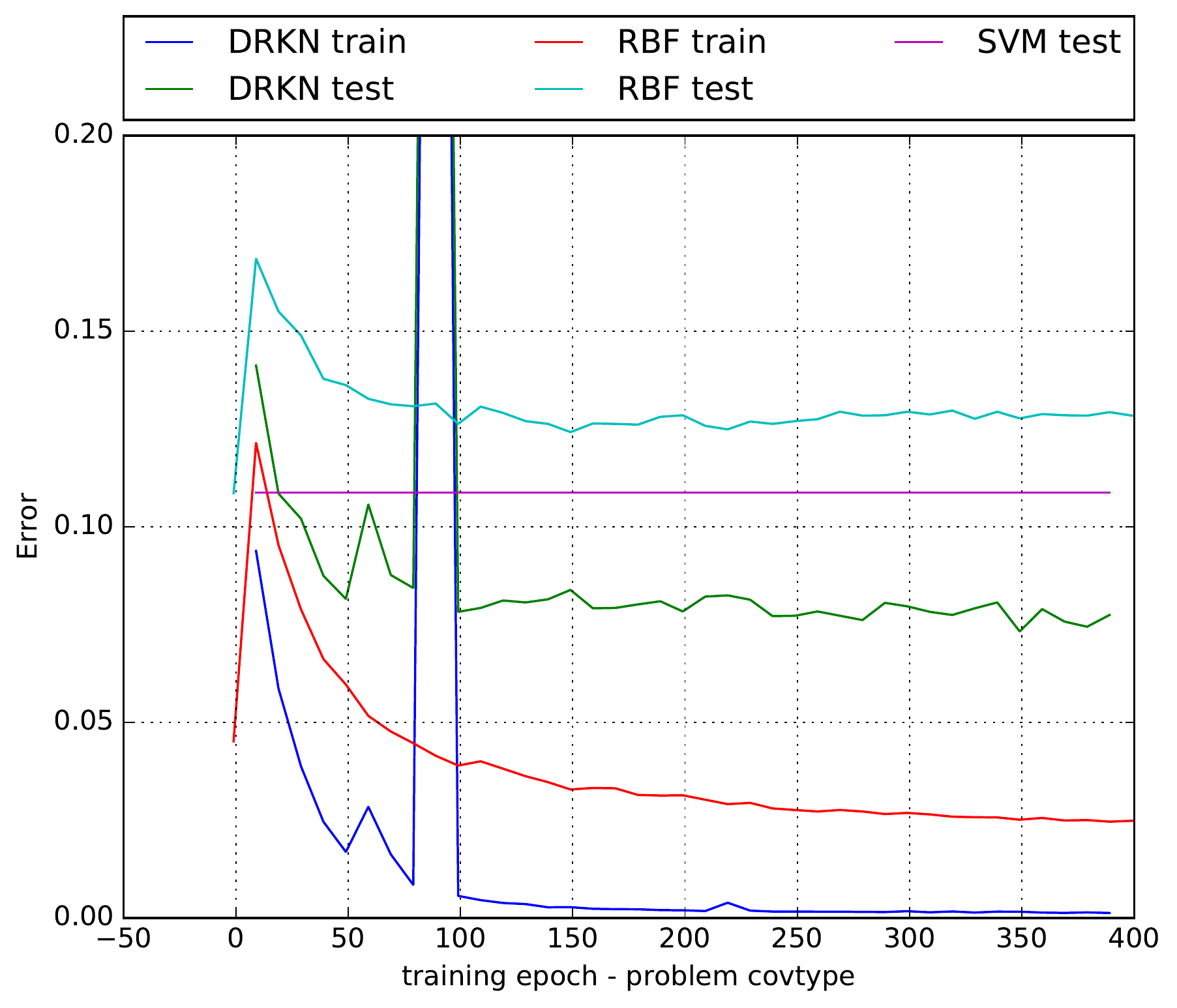}
\caption{Errors and training epoch for each of the example problems. \label{fig-results}} 
\end{figure*}

\section{Discussion}

We have derived a new upper bound for deep networks approximating radially symmetric functions and have shown that deeper networks are more efficient than the 3-layer network of \citet{eldan2016power}. The central concept in this construction is a space fold --- halving the volume of space that each subsequent layer needs to handle. We hypothesise that to take full advantage of deep networks we need to apply operations that work on multiple areas of the input space simultaneously, analogous to taking advantage of the multiple linear regions noted by \citet{montufar2014number}. 
Folding transformations are one way to ensure that any operation applied in a later layer is simultaneously applied to many regions in the input layer. We believe there are many other possible transformations, but the reflections used here might be considered fundamental in a sense due to the Cartan–-Dieudonn\'{e}–-Scherk Theorem which states that all orthogonal transformations can be decomposed into a sequence of reflections. We are yet to fully investigate the consequences of this theorem. 

Radial basis function (RBF) networks are universal approximators and therefore we can use deep radial approximations to approximate any function by decomposing the function into a sum of RBFs first (such as mixture of Gaussians). However, the efficiency of the method will decrease as the number of RBFs required increases. It remains an open question at which point it becomes more efficient to directly approximate the function with a shallow network.

We have used our theoretical construction to
build a deep radially symmetric function approximator and employed it to approximate Gaussian kernel SVMs. Of the 8 problems tested, the method performs much better on 3, moderately better on two, and similarly on 3. In contrast, an actual RBF network approximation has performed no better than the initial SVM.
It should be pointed out that there is no real loss, other than computation, to trying the DRKN on any particular problem. It requires no parameter tuning other than the choice of opimisation method and learning parameters, and if it performs better it can be adopted, but if it does not, one can always fall back to the original SVM.
The flexibility of the DRKN comes at a cost however, as the size of the network can be quite large, and hence training can be slow. However, the method is equally applicable to more scalable approximate SVM algorithms such as the Core Vector Machine \citep{tsang2005core} or SimpleSVM \citep{vishwanathan2002ssvm}, and future work will involve testing the method on these approximate algorithms.

\bibliography{papers}

\begin{thebibliography}{15}
\providecommand{\natexlab}[1]{#1}
\providecommand{\url}[1]{\texttt{#1}}
\expandafter\ifx\csname urlstyle\endcsname\relax
  \providecommand{\doi}[1]{doi: #1}\else
  \providecommand{\doi}{doi: \begingroup \urlstyle{rm}\Url}\fi

\bibitem[Basri and Jacobs(2016)]{basri2016efficient}
Ronen Basri and David Jacobs.
\newblock Efficient representation of low-dimensional manifolds using deep
  networks.
\newblock \emph{arXiv preprint arXiv:1602.04723}, 2016.

\bibitem[Cohen et~al.(2016)Cohen, Sharir, and Shashua]{cohen2016expressive}
Nadav Cohen, Or~Sharir, and Amnon Shashua.
\newblock On the expressive power of deep learning: a tensor analysis.
\newblock \emph{JMLR: Workshop and Conference Proceedings}, 49:\penalty0 1--31,
  2016.

\bibitem[Delalleau and Bengio(2011)]{delalleau2011shallow}
Olivier Delalleau and Yoshua Bengio.
\newblock Shallow vs. deep sum-product networks.
\newblock In \emph{Advances in Neural Information Processing Systems}, pages
  666--674, 2011.

\bibitem[Eldan and Shamir(2016)]{eldan2016power}
Ronen Eldan and Ohad Shamir.
\newblock The power of depth for feedforward neural networks.
\newblock \emph{JMLR: Workshop and Conference Proceedings}, 49:\penalty0
  1–34, 2016.

\bibitem[Fornefett et~al.(2001)Fornefett, Rohr, and
  Stiehl]{fornefett2001radial}
Mike Fornefett, Karl Rohr, and H~Siegfried Stiehl.
\newblock Radial basis functions with compact support for elastic registration
  of medical images.
\newblock \emph{Image and vision computing}, 19\penalty0 (1):\penalty0 87--96,
  2001.

\bibitem[Krizhevsky et~al.(2012)Krizhevsky, Sutskever, and
  Hinton]{krizhevsky2012imagenet}
Alex Krizhevsky, Ilya Sutskever, and Geoffrey~E Hinton.
\newblock Imagenet classification with deep convolutional neural networks.
\newblock In \emph{Advances in neural information processing systems}, pages
  1097--1105, 2012.

\bibitem[Lanckriet et~al.(2004)Lanckriet, Cristianini, Bartlett, Ghaoui, and
  Jordan]{lanckriet2004learning}
Gert~RG Lanckriet, Nello Cristianini, Peter Bartlett, Laurent~El Ghaoui, and
  Michael~I Jordan.
\newblock Learning the kernel matrix with semidefinite programming.
\newblock \emph{Journal of Machine learning research}, 5\penalty0
  (Jan):\penalty0 27--72, 2004.

\bibitem[Montufar et~al.(2014)Montufar, Pascanu, Cho, and
  Bengio]{montufar2014number}
Guido~F Montufar, Razvan Pascanu, Kyunghyun Cho, and Yoshua Bengio.
\newblock On the number of linear regions of deep neural networks.
\newblock In \emph{Advances in Neural Information Processing Systems}, pages
  2924--2932, 2014.

\bibitem[Pedregosa et~al.(2011)Pedregosa, Varoquaux, Gramfort, Michel, Thirion,
  Grisel, Blondel, Prettenhofer, Weiss, Dubourg, Vanderplas, Passos,
  Cournapeau, Brucher, Perrot, and Duchesnay]{scikit-learn}
F.~Pedregosa, G.~Varoquaux, A.~Gramfort, V.~Michel, B.~Thirion, O.~Grisel,
  M.~Blondel, P.~Prettenhofer, R.~Weiss, V.~Dubourg, J.~Vanderplas, A.~Passos,
  D.~Cournapeau, M.~Brucher, M.~Perrot, and E.~Duchesnay.
\newblock Scikit-learn: Machine learning in {P}ython.
\newblock \emph{Journal of Machine Learning Research}, 12:\penalty0 2825--2830,
  2011.

\bibitem[Shaham et~al.(2016)Shaham, Cloninger, and Coifman]{shaham2016provable}
Uri Shaham, Alexander Cloninger, and Ronald~R Coifman.
\newblock Provable approximation properties for deep neural networks.
\newblock \emph{Applied and Computational Harmonic Analysis}, 2016.

\bibitem[Szymanski and McCane(2014)]{szymanski2014deep}
Lech Szymanski and Brendan McCane.
\newblock Deep networks are effective encoders of periodicity.
\newblock \emph{Neural Networks and Learning Systems, IEEE Transactions on},
  25\penalty0 (10):\penalty0 1816--1827, 2014.

\bibitem[Telgarsky(2016)]{telgarsky2016benefits}
Matus Telgarsky.
\newblock Benefits of depth in neural networks.
\newblock \emph{JMLR: Workshop and Conference Proceedings}, 49:\penalty0 1--23,
  2016.

\bibitem[Tsang et~al.(2005)Tsang, Kwok, and Cheung]{tsang2005core}
Ivor~W Tsang, James~T Kwok, and Pak-Ming Cheung.
\newblock Core vector machines: Fast svm training on very large data sets.
\newblock \emph{Journal of Machine Learning Research}, 6\penalty0
  (Apr):\penalty0 363--392, 2005.

\bibitem[Vishwanathan and Murty(2002)]{vishwanathan2002ssvm}
SVM Vishwanathan and M~Narasimha Murty.
\newblock Ssvm: a simple svm algorithm.
\newblock In \emph{Neural Networks, 2002. IJCNN'02. Proceedings of the 2002
  International Joint Conference on}, volume~3, pages 2393--2398. IEEE, 2002.

\bibitem[Wilson et~al.(2016)Wilson, Hu, Salakhutdinov, and
  Xing]{wilson2016deep}
Andrew~Gordon Wilson, Zhiting Hu, Ruslan Salakhutdinov, and Eric~P Xing.
\newblock Deep kernel learning.
\newblock In \emph{Proceedings of the 19th International Conference on
  Artificial Intelligence and Statistics}, pages 370--378, 2016.

\end{thebibliography}
\bibliographystyle{plainnat}

\newpage
\begin{appendices}
\appendixpage
\section{Some Preliminaries}

First a definition of $L$-Lipschitz.
A function is $L$-Lipschitz if:
\begin{equation}
|f(x) - f(y)| \le L |x-y|
\end{equation}

We also need to make use of the following Lemma from \citet{eldan2016power} which we state without proof:
\lemnineteen*

\section{3 Layer Network}
We also need the following lemma which is modified from Lemma 18 of \citep{eldan2016power}. Since it is a modified version, we give a proof.
\lemep*
\begin{proof}
The proof consists of constructing a 3-layer network with the first layer being the input. The second layer approximates $x_i^2$, then the third layer computes $\sum_i x_i^2$ and approximates $f$.

\vspace{1em}
\noindent \emph{Approximate $x_i^2$:}

\noindent Define the $2R$-Lipschitz function:
\begin{equation}
l(x) = \min\{x^2, R^2\},
\end{equation}

\noindent Using Lemma \ref{lem19}, create the function $\bar{l}(x)$ having the form $a + \sum_{i=1}^{w_1} \alpha_i \sigma(x-\beta_i)$ so that
\begin{equation}
\sup_{x \in \myreal} \left| \bar{l}(x)-l(x) \right| \le \frac{\delta}{d}
\end{equation}
where $w_1\le \frac{6 d R^2}{\delta}$.

\vspace{1em}
\noindent \emph{Appproximate $\sum_i x_i^2$:}

\noindent Define the function $\ell: \myreal^d \to \myreal$
\begin{equation}
\ell(\myvec{x}) = \sum_{i=1}^d l(x_i) = \sum_{i=1}^d \min\{x_i^2, R^2\}
\end{equation}
$\ell(\myvec{x})$ is $2dR$-Lipschitz.
Consequently, define the function
\begin{equation}
\bar{\ell}(\myvec{x}) = \sum_{i=1}^d \bar{l}(x_i) = \sum_{i=1}^d \left[ a_i + \sum_{j=1}^{w_1} \alpha_{ij} \sigma(x_i-\beta_{ij}) \right]
\end{equation}

\noindent Note that $\bar{\ell}(\myvec{x})$ is $2d R$-Lipschitz.
At this point we have $w_{12} = d w_1 \le \frac{6d^2 R^2}{\delta}$ weights in the first two layers of the network and:
\begin{equation}
\sup_{\myvec{x} \in \myreal^d} \left|\bar{\ell}(\myvec{x})-\ell(\myvec{x}) \right| \le \delta.
\end{equation}

\noindent \emph{Approximate $f(||\myvec{x}||)$}:

\noindent The input to the final layer is $\bar{\ell}(\myvec{x})$ which is an approximation of $\ell(\myvec{x})=||\myvec{x}||^2$. 
The error associated with approximating $f(\sqrt{\ell(\myvec{x})})$ is:
\begin{align}
|f(\sqrt{\bar{\ell}(\myvec{x})}) - f(\sqrt{\ell(\myvec{x})})| &\le L | \sqrt{\ell(\myvec{x})\pm\delta} - \sqrt{\ell({\myvec{x}})}|\\
&\le L | \sqrt{\ell(\myvec{x})} \pm \sqrt{\delta} - \sqrt{\ell({\myvec{x}})}|\\
&\le L | \sqrt{\delta} |
\end{align}
Since $f$ is L-Lipschitz, $\ell(\myvec{x})-\delta\le\bar{\ell}(\myvec{x})\le\ell(\myvec{x})+\delta$ and $\sqrt{\ell(\myvec{x})\pm\delta}\le \sqrt{\ell(\myvec{x})} \pm \sqrt{\delta}$.

\noindent Now we are able to approximate $f(\sqrt{\bar{\ell}(\myvec{x})})$ using Lemma \ref{lem19} with a function of the form:
\begin{equation}
g(\myvec{x}) = a + \sum_{k=1}^{w_3} \alpha_k \sigma \left(\sum_{i=1}^d \left[ a_i + \sum_{j=1}^{w_1} \alpha_{ij} \sigma(x_i-\beta_{ij}) \right] - \beta_k \right)
\end{equation}

From Lemma \ref{lem19}:
\begin{align}
\sup_{\myvec{x} \in \myreal^d}|g(\myvec{x}) - f(\sqrt{\bar{\ell}(\myvec{x})})| &\le \delta
\end{align}
with $w_3 \le \frac{3RL}{\delta}$. So either (taking the worst case error):
\begin{align}
\sup_{\myvec{x} \in \myreal^d}|g(\myvec{x}) - f(\sqrt{\ell(\myvec{x})})| & =\sup_{\myvec{x} \in \myreal^d}|g(\myvec{x}) - f(\sqrt{\bar{\ell}(\myvec{x})}) - L\sqrt{\delta}| \\
& = \sup_{\myvec{x} \in \myreal^d} ||g(\myvec{x}) - f(\sqrt{\bar{\ell}(\myvec{x})})| - L\sqrt{\delta}| \\
& \le |\delta - L\sqrt{\delta}|
\end{align}
or
\begin{align}
\sup_{\myvec{x} \in \myreal^d}\left| g(\myvec{x}) - f(\sqrt{\ell(\myvec{x})}) \right| & =\sup_{\myvec{x} \in \myreal^d} \left| g(\myvec{x}) - f(\sqrt{\bar{\ell}(\myvec{x})}) + L\sqrt{\delta} \right| \\
&= \sup_{\myvec{x} \in \myreal^d} \left| g(\myvec{x}) - f(\sqrt{\bar{\ell}(\myvec{x})}) \right| + L\sqrt{\delta}\\
&\le \delta + L\sqrt{\delta}
\end{align}
and therefore:
\begin{equation}
\sup_{\myvec{x} \in \myreal^d} \left| g(\myvec{x}) - f(\sqrt{\ell(\myvec{x})}) \right| \le \delta + L\sqrt{\delta}
\end{equation}
and the number of weights is at most: $\frac{6 d^2 R^2 + 3 R L}{\delta}$.
\end{proof}

\section{Folding Transformations}
\label{sec-app-folding}

In this section we show how folding transformations \citep{szymanski2014deep} can be used to create a much deeper network with the same error, but many fewer weights than needed in Lemma \ref{lem-3layer}. A folding transformation is one in which half of a space is reflected about a hyperplane, and the other half remains unchanged. Figure \ref{fig-fold-2d} demonstrates how a sequence of folding transformations can transform a circle in 2D to a small sector.
We will use this general idea to prove the following theorem:
\main*

The approach taken here is a constructive one and specifies the architecture of the network needed to approximate a function. In fact, all of the weights except those in the last layer are specified. The approach is somewhat different to that used to prove Lemma \ref{lem-3layer}. We build a sequence of layers to directly approximate $||\myvec{x}||$ and then use Lemma \ref{lem19} to approximate $f$. To build our layers, we need a few helper lemmas.

\lemtwodfold*
\begin{proof}
The necessary ReLU network is shown in Figure \ref{fig-relu-fold-2d}. Only one of the nodes labeled $x_{-}$ ($y_{-}$) and $x_{+}$ ($y_{+}$) are active at any one time. Therefore there are four possible cases depending on which two nodes are active. Note that $x_{-}$ is active when $\myvec{l} \cdot \myvec{x^\perp} < 0$ and $x_{+}$ is active when $\myvec{l} \cdot \myvec{x^\perp} > 0$.

\noindent \emph{Case 1}: $x_{-}$ and $y_{-}$
\noindent 
\begin{align*}
x' &= l_y(-l_y x + l_x y) - l_x(-l_x x - l_y y)\\
y' &= -l_x(-l_y x + l_x y) - l_y (-l_x x - l_y y)\\
\begin{bmatrix}
x' \\ y'
\end{bmatrix}
&=
\begin{bmatrix}
l_x^2 -l_y^2 & 2 l_x l_y\\
2 l_x l_y & l_y^2 - l_x^2
\end{bmatrix}
\begin{bmatrix}
x \\ y
\end{bmatrix}
\end{align*}

\noindent \emph{Case 2}: $x_{+}$ and $y_{-}$
\noindent
\begin{align*}
x' &= l_y (l_y x - l_x y) -l_x(-l_x x -l_y y) \\
y' &= -l_x(l_y x - l_x y) -l_y(-l_x x - l_y y)\\
\begin{bmatrix}
x' \\ y'
\end{bmatrix}
&=
\begin{bmatrix}
x \\ y
\end{bmatrix}
\end{align*}

\noindent \emph{Case 3}: $x_{-}$ and $y_{+}$
\noindent
\begin{align*}
x' &= l_y(-l_y x + l_x y) + l_x (l_x x + l_y y)\\
y' &= -l_x(-l_y x + l_x y) + l_y (l_x x + l_y y)\\
\begin{bmatrix}
x' \\ y'
\end{bmatrix}
&=
\begin{bmatrix}
l_x^2 -l_y^2 & 2 l_x l_y\\
2 l_x l_y & l_y^2 - l_x^2
\end{bmatrix}
\begin{bmatrix}
x \\ y
\end{bmatrix}
\end{align*}

\noindent \emph{Case 4}: $x_{+}$ and $y_{+}$
\noindent
\begin{align*}
x' &= l_y (l_y x - l_x y) +l_x(l_x x +l_y y) \\
y' &= -l_x(l_y x - l_x y) +l_y(l_x x + l_y y)\\
\begin{bmatrix}
x' \\ y'
\end{bmatrix}
&=
\begin{bmatrix}
x \\ y
\end{bmatrix}
\end{align*}

\end{proof}

\lemapproxtwod*
\begin{proof}

We simply stack layers of the type shown in Figure \ref{fig-relu-fold-2d} with suitable choice of $l_x, l_y$ at each layer. Note that the summation nodes aren't required since they can be incorporated into the summations and weights of the next ReLU layer. 
Call the first ReLU layer, layer 1, and set $l_{x,i}=\cos(\frac{\pi}{2^{i-1}})$, $l_{y,i}=\sin(\frac{\pi}{2^{i-1}})$. Layer 1 will fold all points to the positive $y$-axis half-plane. Layer 2 will fold all points to the positive $(x, y)$-quadrant. Layer 3 to within $\frac{\pi}{4}$ of the $x$-axis etc. After $f$ such layers, $||\myvec{x}||$ can be approximated with the resulting $x$-coordinate, with the following error:
\begin{align}
\delta_1 &=||\myvec{x}||-\hat{x}\\
&= ||\myvec{x}||-||\myvec{x}||\cos(\frac{\pi}{2^{f}})\\
&\le R(1-\cos(\frac{\pi}{2^{f}}))\\
&\le R(2 \sin^2(\frac{\pi}{2^{f+1}}))\\
&\le R(2 \sin(\frac{\pi}{2^{f+1}}))\\
&\le R(\frac{\pi}{2^{f}})\\
2^f &\le R \frac{\pi}{\delta_1}\\
f &\le \log_2(R \frac{\pi}{\delta_1})
\end{align}

\end{proof}

\lemapproxvecx*
\begin{proof}
We note that a fold in 2D plane in $\myreal^d$ will leave all coordinates perpendicular to the plane unchanged. We can therefore apply the approximation of Lemma \ref{lem-approx-2d} to pairs of input coordinates to produce $d/2$ new coordinates. Then apply the same reduction to produce $d/4$ coordinates and continue on this way until there is only one coordinate left. In effect, we are calculating the norm via the following scheme:
\begin{equation}
\sqrt{x_1^2 + x_2^2 + x_3^2 + \cdots + x_d^2} =
\sqrt{ \cdots \sqrt{ \sqrt{x_1^2 + x_2^2}^2 + \sqrt{x_3^2+x_4^2}^2}^2 \cdots \sqrt{\sqrt{\cdots}^2 + \sqrt{x_{n-1}^2 + x_{n-2}^2}^2}^2}
\end{equation}

Let $g_i$ represent the function mapping of fold sequence $i$.
That is, $g_{1,x_{1:2}}$ denotes the first fold sequence that folds the first two coordinates $x_1, x_2$; $g_{1,x_{3:4}}$ folds the second two coordinates etc. And $g_{2,x_{1:4}}$ denotes the second fold sequence that folds the output of the first two folds in layer one. See Figure \ref{fig-multiple-folds} for a schematic of the network. More formally, we have (with some abuse of notation):
\begin{align*}
g_{1,j:j+1} &= g_1(x_j,x_{j+1}) \\
g_{i>1,j:j+2^i-1} &= g_i(g_{i-1,x_{j:j+2^{i-1}-1}},g_{i=1,x_{j+2^{i-1}:j+2^i-1}})
\end{align*}
Each fold layer requires $\log_2(R \frac{\pi}{\delta_1})$ network layers of 4 neurons each and results in an error no greater than $\delta_1$. We have the following situation after the first fold layer:
\begin{align*}
\sqrt{x_1^2+x_2^2}-\delta_1 \le g_{1,x_{1:2}} &\le \sqrt{x_1^2+x_2^2}+\delta_1\\
\sqrt{x_3^2+x_4^2}-\delta_1 \le g_{1,x_{3:4}} &\le \sqrt{x_3^2+x_4^2}+\delta_1\\
&\cdots
\end{align*}
We proceed via induction and bound the error produced at each subsequent layer:
\begin{align}
&\sqrt{\sum_{j=n}^{n+2^{i}-1} x_j^2} - \left[\lemfive{(i+1)}{i} \right] \delta_1 \nonumber\\ 
&\le g_{i,x_{n:n+2^{i}-1}}
\nonumber \\
&\le \sqrt{\sum_{j=n}^{n+2^{i}-1} x_j^2} + \left[ \lemfive{(i+1)}{i} \right] \delta_1,
\label{eqn-error-induct}
\end{align}
where $n$ is the appropriate coordinate index, $2\le i \le d$, $\delta_1$ is the error from a single fold. Appropriate coordinate index in this context means $n \in \{1,5,9,\cdots\}$ for $i=2$, $n \in \{1,9,17,\cdots\}$ for $i=3$, etc.

At the second fold layer (the base case, $i=2$), we would like to bound the error on the result in terms of the target computation ($\sqrt{x_1^2+x_2^2+x_3^2+x_4^2})$. We use the first 4 coordinates without loss of generality. The second fold layer gives us:
\begin{align*}
\sqrt{g_{1,x_{1:2}}^2 + g_{1,x_{3:4}}^2} - \delta_1 &\le g_{2,x_{1:4}}
\le \sqrt{g_{1,x_{1:2}}^2 + g_{1,x_{3:4}}^2} + \delta_1
\end{align*}
First, consider the right hand side:
\begin{align*}
&g_{2,x_{1:4}} \\
&\le \sqrt{g_{1,x_{1:2}}^2 + g_{1,x_{3:4}}^2} + \delta_1\\
&\le \sqrt{(\sqrt{x_1^2+x_2^2}+\delta_1)^2 + (\sqrt{x_3^2+x_4^2}+\delta_1)^2 } + \delta_1\\
&\le \sqrt{x_1^2+x_2^2 + x_3^2 + x_4^2 +2 \delta_1(\sqrt{x_1^2+x_2^2} + \sqrt{x_3^2+x_4^2}) + 2\delta_1^2 } + \delta_1\\
&\le \sqrt{x_1^2+x_2^2 + x_3^2 + x_4^2 + 2\sqrt{2} \delta_1 (\sqrt{x_1^2+x_2^2+x_3^2+x_4^2}) + 2 \delta_1^2 } + \delta_1\\
&\le \sqrt{(\sqrt{x_1^2+x_2^2+x_3^2+x_4^2} + \sqrt{2}\delta_1)^2} + \delta_1\\
&\le \sqrt{x_1^2+x_2^2+x_3^2+x_4^2} + (1+\sqrt{2}) \delta_1
\end{align*}


The left hand side:
\begin{align*}
&g_{2,x_{1:4}}\\
&\ge \sqrt{g_{1,x_{1:2}}^2 + g_{1,x_{3:4}}^2} - \delta_1\\
&\ge \sqrt{(\sqrt{x_1^2+x_2^2}-\delta_1)^2 + (\sqrt{x_3^2+x_4^2}-\delta_1)^2 } - \delta_1\\
&\ge \sqrt{x_1^2+x_2^2 + x_3^2 + x_4^2 - 2\delta_1(\sqrt{x_1^2+x_2^2}+\sqrt{x_3^2+x_4^2}) + 2\delta_1^2 } - \delta_1\\
&\ge \sqrt{x_1^2+x_2^2 + x_3^2 + x_4^2 - 2\sqrt{2}\delta_1(\sqrt{x_1^2+x_2^2+x_3^2+x_4^2}) + 2\delta_1^2} - \delta_1 \\
&\ge \sqrt{ (\sqrt{x_1^2+x_2^2+x_3^2+x_4^2} - \sqrt{2}\delta_1)^2 } - \delta_1\\
&\ge \sqrt{x_1^2+x_2^2+x_3^2+x_4^2} - (1+\sqrt{2})\delta_1
\end{align*}
These clearly satisfy Equation \ref{eqn-error-induct}.

For the induction step we want to show that Equation \ref{eqn-error-induct} is true for $i+1$, given it is true for $i$. We start by applying Lemma \ref{lem-2d-fold}:
\begin{align}
&\sqrt{g_{i,x_{n:n+2^i-1}}^2+g_{i,x_{n+2^i:n+2^{i+1}-1}}^2}-\delta_1 \nonumber\\
&\le g_{i+1,x_{n:n+2^{i+1}-1}} \nonumber\\
&\le \sqrt{g_{i,x_{n:n+2^i-1}}^2+g_{i,x_{n+2^i:n+2^{i+1-1}}}^2}+\delta_1
\label{eq-ind-s1}
\end{align} 

\noindent Start with the right hand side and for brevity just use $g_{i+1}$, and let $\delta_i =  \left[ \lemfive{(i+1)}{i} \right] \delta_1$, $a=\sum_{j=n}^{n+2^i-1} x_j^2$, $b=\sum_{j=n+2^i}^{n+2^{i+1}-1} x_j^2$:
\begin{align*}
g_{i+1}
&\le \sqrt{\left(\sqrt{a}+\delta_i\right)^2 + \left(\sqrt{b}+\delta_i\right)^2} + \delta_1\\
&\le \sqrt{a + b + 2 \delta_i (a+b) + 2\delta_i^2} + \delta_1 \\
&\le \sqrt{ (\sqrt{a+b} + \sqrt{2}\delta_i)^2} + \delta_1 \\
&\le \sqrt{ a + b } + \sqrt{2}\delta_i + \delta_1\\
&\le \sqrt{a + b} + \sqrt{2} \left[ \lemfive{(i+1)}{i} \right] \delta_1 + \delta_1\\
&\le \sqrt{a + b} + \left[ \sqrt{2}(\lemfive{(i+1)}{i}) + 1 \right] \delta_1 \\
&\le \sqrt{a + b} + \left[ \lemfive{(i+2)}{(i+1)} \right] \delta_1 \\
\end{align*}
As required. A similar argument can be constructed for the left hand side.

\noindent Each folding sequence requires $\log_2(R \frac{\pi}{\delta_1})$ layers and results in an error no greater than $\delta_1$. We need $\log_2{d}$ such folding sequences giving the number of layers as:
\begin{equation}
N_l = \log_2{d} \log_2(R \frac{\pi}{\delta_1})
\end{equation}
The first fold sequence requires $2d$ neurons per layer and $4d$ weights per layer. The second fold sequence requires $d$ neurons per layer and $2d$ weights per layer and so on. The total number of neurons then is:
\begin{align*}
N_n&= \sum_{i=1}^{\log_2{d}} \frac{2d}{2^{i-1}} \log_2{(R \frac{\pi}{\delta_1})}\\
&= 4(d-1) \log_2{(R \frac{\pi}{\delta_1})}
\end{align*}
The total number of weights, then is:
\begin{equation*}
N_w = 8(d-1)\log_2{(R \frac{\pi}{\delta_1})}
\end{equation*}

\noindent In terms of the overall error of the approximation ($\delta$), we have:
\begin{align*}
N_l&=\log_2{d} \log_2{\left(R \left[\lemfive{(d+1)}{d} \right]\frac{\pi}{\delta}\right)}\\
&=O\left(d \log_2{d} + d \log_2{\left( R \frac{\pi}{\delta} \right)} \right)\\
N_n&=4(d-1) \log_2{\left(R \left[\lemfive{(d+1)}{d} \right]\frac{\pi}{\delta}\right)}\\
&= O\left(d^2 + d \log_2{\left(R \frac{\pi}{\delta} \right)}\right)\\
N_w&= O\left(d^2  + d \log_2{\left(R \frac{\pi}{\delta} \right)}\right)
\end{align*}
\end{proof}

Now we are ready to prove the main theorem.

\main*
\begin{proof}
From Lemma \ref{lem-approx-vecx} we can approximate $||\myvec{x}||$ to within $\sqrt{\delta}$. Therefore:
\begin{align*}
f(||\myvec{x}||+\sqrt{\delta})-\delta &\le g(||\myvec{x}||+\sqrt{\delta}) \le f(||\myvec{x}||+\sqrt{\delta})+\delta, \text{  from Lemma \ref{lem19}}\\
f(||\myvec{x}||)+L\sqrt{\delta}-\delta &\le g(||\myvec{x}||+\sqrt{\delta}) \le f(||\myvec{x}||)+L\sqrt{\delta}+\delta
\end{align*}
and 
\begin{align*}
f(||\myvec{x}||-\sqrt{\delta})-\delta &\le g(||\myvec{x}||-\sqrt{\delta}) \le f(||\myvec{x}||-\sqrt{\delta})+\delta\\
f(||\myvec{x}||)-L\sqrt{\delta}-\delta &\le g(||\myvec{x}||+\sqrt{\delta}) \le f(||\myvec{x}||)-L\sqrt{\delta}+\delta
\end{align*}
therefore:
\begin{equation*}
f(||\myvec{x}||)-L\sqrt{\delta}-\delta \le g(||\myvec{x}||+\sqrt{\delta}) \le f(||\myvec{x}||)+L\sqrt{\delta}+\delta
\end{equation*}

\noindent The number of weights and neurons required by Lemma \ref{lem19} is $3 \frac{RL}{\delta}$. The number of weights and neurons required to estimate $||\myvec{x}||$ is given by Lemma \ref{lem-approx-vecx} (substituting $\sqrt{\delta}$ for $\delta$). Stack the network from Lemma \ref{lem19} onto the end of the network from Lemma \ref{lem-approx-vecx}, thus requiring a total number of neurons no more than:
\begin{align*}
N_n &\le \left[ 4(d-1) \log_2{\left(\frac{R\pi}{\sqrt{\delta}} \left[ \lemfive{(d-1)}{d} \right]\right)}\right] + \frac{3RL}{\delta} \\
N_n &= O(d^2 + d\log_2(\frac{R}{\sqrt{\delta}}) + \frac{3 RL}{\delta})
\end{align*}
and a total number of weights no more than:
\begin{align*}
N_w &\le \left[ 8(d-1) 
\log_2{\left(\frac{R\pi}{\sqrt{\delta}} \left[ \lemfive{(d-1)}{d} \right]\right)}
\right] + \frac{3RL}{\delta}\\
N_w &= O(d^2 + d \log_2(\frac{R}{\sqrt{\delta}}) + \frac{3 RL}{\delta})
\end{align*}
\end{proof}

\end{appendices}

\end{document}